\documentclass[journal,onecolumn,12pt]{IEEEtran}
\date{}

\usepackage{amsfonts}
\usepackage{amsmath}
\usepackage{amssymb}
\usepackage{amsthm}
\usepackage{array}
\usepackage{caption}
\usepackage{cite}
\usepackage{citesort}
\usepackage{color}
\usepackage{epsfig}
\usepackage{enumitem} 
\usepackage{flexisym}
\usepackage{graphicx}
\usepackage{hyphenat}
\usepackage{ifpdf}
\usepackage{latexsym}
\usepackage{mathtools}
\usepackage{multirow}
\usepackage{psfrag}
\usepackage{setspace}
\usepackage{subfigure}
\usepackage{tikz}
\usepackage{times}
\usepackage{tkz-berge}
\usepackage{xparse}

\newtheorem{theorem}{Theorem}
\newtheorem{definition}{Definition}

\newtheorem{lemma}{Lemma}

\newtheorem{remark}{Remark}
\newtheorem{example}{Example}

\title{Fundamental Conditions for Low-CP-Rank Tensor Completion}

\author{Morteza Ashraphijuo and Xiaodong Wang \thanks{The authors are with the Department of Electrical Engineering, Columbia University, NY, email: \{ashraphijuo,wangx\}@ee.columbia.edu.}}
 
\begin{document}
\maketitle

\begin{abstract}

We consider the problem of low canonical polyadic (CP) rank tensor completion. A completion is a tensor whose entries agree with the observed entries and its rank matches the given CP rank. We analyze the manifold structure corresponding to the tensors with the given rank and define a set of polynomials based on the sampling pattern and CP decomposition. Then, we show that finite completability of the sampled tensor is equivalent to having a certain number of algebraically independent polynomials among the defined polynomials. Our proposed approach results in characterizing the maximum number of algebraically independent polynomials in terms of a simple geometric structure of the sampling pattern, and therefore we obtain the deterministic necessary and sufficient condition on the sampling pattern for finite completability of the sampled tensor. Moreover, assuming that the entries of the tensor are sampled independently with probability $p$ and using the mentioned deterministic analysis, we propose a combinatorial method to derive a lower bound on the sampling probability $p$, or equivalently, the number of sampled entries that guarantees finite completability with high probability. We also show that the existing result for the matrix completion problem can be used to obtain a loose lower bound on the sampling probability $p$. In addition, we obtain deterministic and probabilistic conditions for unique completability. It is seen that the number of samples required for finite or unique completability obtained by the proposed analysis on the CP manifold is orders-of-magnitude lower than that is obtained by the existing analysis on the Grassmannian manifold.

\

\begin{IEEEkeywords}
Low-rank tensor completion, canonical polyadic (CP) decomposition, finite completability, unique completability, algebraic geometry, Bernstein's theorem.
\end{IEEEkeywords}

\end{abstract}

\newpage

\section{Introduction}

The fast progress in data science and technology has given rise to the extensive applications of multi-way datasets, which allow us to take advantage of the inherent correlations across different attributes. The classical matrix analysis limits its efficiency in exploiting the correlations across different features from a multi-way perspective. In contrast, analysis of multi-way data (tensor), which was originally proposed in the field of psychometrics and recently found applications in machine learning and signal processing, is capable of taking full advantage of these correlations \cite{de20,kolda2,abraham2,gras,muti,signo}. The problem of low-rank tensor completion, i.e., reconstructing a tensor from a subset of its entries given the rank, which is generally NP hard \cite{hilt}, arises in compressed sensing \cite{lim,sid,gandy}, visual data reconstruction  \cite{visual,liulow2}, seismic data processing \cite{kreimer,ely20135d,wang2016tensor}, etc.. Existing approaches to low-rank data completion mainly focus on convex relaxation of matrix rank \cite{candes,candes2,cai,ashraphijuo2016c,phase} or different convex relaxations of tensor ranks \cite{gandy,tomioka,nuctensor,romera}.

Tensors consisting of real-world datasets usually have a low rank structure. The manifold of low-rank tensors has recently been investigated in several works \cite{gras,ashraphijuo,ashraphijuo3}. In this paper, we focus on the canonical polyadic (CP) decomposition \cite{harstions,ten,Eck,kruskal} and the corresponding CP rank, but in general there are other well-known tensor  decompositions including Tucker decomposition \cite{Tuck,SVD,Tuckermanifold}, tensor-train (TT) decomposition \cite{oseledets,holtz}, tubal rank decomposition \cite{kilmer2013third} and several other methods \cite{de2,papa}. Note that most existing literature on tensor completion based on various optimization formulations use CP rank \cite{gandy,low2}.

In this paper, we study the fundamental conditions on the sampling pattern to ensure finite or unique number of completions, where these fundamental conditions are independent of the correlations of the entries of the tensor, in contrast to the common assumption adopted in literature such as incoherence. Given the rank of a matrix, Pimentel-Alarc{\'o}n {\it et. al.} in \cite{charact} obtains such fundamental conditions on the sampling pattern for finite completability of the matrix. Previously, we treated the same problem for multi-view matrix \cite{ashraphijuo2}, tensor given its Tucker rank \cite{ashraphijuo}, and tensor given its TT rank\cite{ashraphijuo3}. In this paper, the structure of the CP decomposition and the geometry of the corresponding manifold are investigated to obtain the fundamental conditions for finite completability given its CP rank.

To emphasize the contribution of this work, we highlight the differences and challenges in comparison with the Tucker and TT tensor models. In CP decomposition, the notion of  tensor multiplication is different from those for  Tucker and TT, and therefore the geometry of the manifold and the algebraic variety are completely different. Moreover, in CP decomposition we are dealing with the sum of several tensor products, which is not the case in Tucker and TT decompositions, and therefore the equivalence classes or geometric patterns that are needed to study the algebraic variety are different. Moreover, CP rank is a scalar and the ranks of matricizations and unfoldings are not given in contrast with the Tucker and TT models.

Let $\mathcal{U}$ denote the sampled tensor and $\Omega$ denote the binary sampling pattern tensor that is of the same dimension and size as $\mathcal{U}$. The entries of $\Omega$ that correspond to the observed entries of $\mathcal{U}$ are equal to $1$ and the rest of the entries are set as $0$. Assume that the entries of $\mathcal{U}$ are sampled independently with probability $p$. This paper is mainly concerned with treating the following three problems.

{\bf Problem (i):} Given the CP rank, characterize the necessary and sufficient conditions on the sampling pattern $\Omega$, under which there exist only finitely many completions of $\mathcal{U}$.

We consider the CP decomposition of the sampled tensor, where all rank-$1$ tensors in this decomposition are unknown and we only have some entries of $\mathcal{U}$. Then, each sampled entry results in a polynomial such that the variables of the polynomial are the entries of the rank-$1$ tensors in the CP decomposition. We propose a novel analysis on the CP manifold to obtain the maximum number of algebraically independent polynomials, among all  polynomials corresponding to the sampled entries, in terms of the geometric structure of the sampling pattern $\Omega$. We show that if the maximum number of algebraically independent polynomials is a given number, then the sampled tensor $\mathcal{U}$ is finitely completable. Due to the fundamental differences between the CP decomposition and the Tucker or TT decomposition, this analysis is completely different from our previous works \cite{ashraphijuo,ashraphijuo3}. Moreover, note that our proposed algebraic geometry analysis on the CP manifold is not a simple generalization of the existing analysis on the Grassmannian manifold \cite{charact} even though the CP decomposition is a generalization of rank factorization of a matrix, as almost every step needs to be developed anew.

{\bf Problem (ii):} Characterize conditions on the sampling pattern to ensure that there is exactly one completion for the given CP rank.

Similar to Problem (i), our approach is to study the algebraic independence of the polynomials corresponding to the samples. We exploit the properties of a set of minimally algebraically dependent polynomials to add additional constraints on the sampling pattern such that each of the rank-$1$ tensors in the CP decomposition can be determined uniquely.

{\bf Problem (iii):} Provide a lower bound on the total number of sampled entries or the sampling probability $p$ such that the proposed conditions on the sampling pattern $\Omega$ for finite and unique completability are satisfied with high probability.

We develop several combinatorial tools together with our previous graph theory results in \cite{ashraphijuo} to obtain lower bounds on the total number of sampled entries, i.e., lower bounds on the sampling probability $p$, such that the deterministic conditions for Problems (i) and (ii) are met with high probability. Particularly, it is shown in \cite{low2}, $\mathcal{O}(n r^{\frac{d-1}{2}}d^2 \log(r))$ samples are required to recover the tensor $\mathcal{U} \in \mathbb{R}^{\overbrace {n \times \dots \times n}^{d}}$ of rank $r$. Recall that in this paper, we obtain the number samples to ensure finite/unique completability that is independent of the completion algorithm. As we show later, using the existing analysis on the Grassmannian manifold results in $\mathcal{O}(n^{\frac{d+1}{2}} \max \left\{ d \log(n) + \log(r), r \right\})$ samples to ensure finite/unique completability. However, our proposed analysis on the CP manifold results in $\mathcal{O}(n^2 \max \left\{  \log(nrd) , r \right\})$ samples to guarantee the finiteness of the number of completions, which is significantly lower than that given in \cite{low2}. Hence, the fundamental conditions for tensor completion motivate new optimization formulation to close the gap in the number of required samples.

The remainder of this paper is organized as follows. In Section \ref{notationsx}, the preliminaries and problem statement are presented. In Section \ref{sec2}, we develop necessary and sufficient deterministic conditions for finite completability. In Section \ref{sec3}, we develop probabilistic conditions for finite completability. In Section \ref{secuni}, we consider unique completability and obtain both deterministic and probabilistic conditions. Some numerical results are provided in Section \ref{simsec}. Finally, Section \ref{concsec} concludes the paper.

\section{Background}\label{notationsx}
\subsection{Preliminaries and Notations}\label{notations}

In this paper, it is assumed that a $d$-way tensor $\mathcal{U} \in \mathbb{R}^{n_1  \times \cdots \times n_d}$ is sampled. Throughout this paper, we use CP rank as the rank of a tensor, which is defined as the minimum number $r$ such that there exist $\mathbf{a}_{i}^{l} \in \mathbb{R}^{n_i}$ for $1 \leq i \leq d$ and $1 \leq l \leq r$ and
\begin{eqnarray}\label{CPdecom}
\mathcal{U} = \sum_{l=1}^{r}  \mathbf{a}_{1}^{l} \otimes \mathbf{a}_{2}^{l} \otimes \dots \otimes \mathbf{a}_{d}^{l},
\end{eqnarray}
or equivalently,
\begin{eqnarray}\label{CPdecompolyeq}
\mathcal{U}(x_1,x_2,\dots,x_d) = \sum_{l=1}^{r} \mathbf{a}_{1}^{l}(x_1)  \mathbf{a}_{2}^{l}(x_2) \dots  \mathbf{a}_{d}^{l}(x_d),
\end{eqnarray}
where $\otimes$ denotes the tensor product (outer product) and $\mathcal{U}(x_1,x_2,\dots,x_d)$ denotes the entry of the sampled tensor with coordinates $ \vec{x} = (x_1,x_2,\dots,x_d)$ and $\mathbf{a}_{i}^{l}(x_i) $ denotes the $x_i$-th entry of vector $\mathbf{a}_{i}^{l}$. Note that $\mathbf{a}_{1}^{l} \otimes \mathbf{a}_{2}^{l} \otimes \dots \otimes \mathbf{a}_{d}^{l} \in \mathbb{R}^{n_1  \times \cdots \times n_d}$ is a rank-$1$ tensor, $l=1,2,\dots,r$.



Denote $\Omega$ as the binary sampling pattern tensor that is of the same size as $\mathcal{U}$ and $\Omega(\vec{x})=1$ if $\mathcal{U}(\vec{x})$ is observed and  $\Omega(\vec{x})=0$ otherwise. Also define $x^+\triangleq\max\{0,x\}$.

For a nonempty  $I \subset \{1,\dots,d\} $, define $N_{I} \triangleq  \Pi_{i \in I} \  n_i $ and also denote $\bar I \triangleq \{1,\dots,d\} \backslash I$. Let $\mathbf{\widetilde U}_{(I)} \in \mathbb{R}^{N_I\times  N_{\bar {I}}}$ be the  unfolding of the tensor $\mathcal{U}$ corresponding to the set $I$ such that $\mathcal{U}(\vec{x}) = $ \\ $\mathbf{\widetilde U}_{(I)}({\widetilde  M}_{I} (x_{i_1},\dots,x_{i_{|I|}}),{\widetilde{{M}} }_{\bar I} (x_{i_{|I|+1}},\ldots,x_{i_d}))$, where $I=\{{i_1},\dots,{i_{|I|}}\}$, $\bar I=\{{i_{|I|+1}},\dots,{i_d}\}$, ${\widetilde M}_{I}: (x_{i_1},\dots,x_{i_{|I|}}) $ $ \rightarrow  \{1,2,\dots, N_I\}$ and ${\widetilde{{M}}}_{\bar I}: (x_{i_{|I|+1}},\ldots,x_{i_d}) \rightarrow  \{1,2,\dots, \bar N_{\bar I} \}$ are two bijective mappings. For $i \in \{1,\dots,d\}$ and $I=\{i\}$, we denote the unfolding corresponding to $I$ by $\mathbf{U}_{(i)}$ and we call it the $i$-th matricization of tensor $\mathcal{U}$.



\subsection{Problem Statement and A Motivating Example}

We are interested in finding necessary and sufficient deterministic conditions on the sampling pattern tensor $\Omega$ under which there are infinite, finite, or unique completions of the sampled tensor $\mathcal{U}$ that satisfy the given CP rank constraint. Furthermore, we are interested in finding probabilistic conditions on the number of samples or the sampling probability that ensure the obtained deterministic conditions for finite and unique completability hold with high probability.

To motivate our proposed analysis in this paper on the CP manifold, we compare  the following two approaches using a simple example to emphasize the exigency of our proposed analysis: (i) analyzing each of the unfoldings individually, (ii) analyzing based on the CP decomposition.

Consider a three-way tensor $\mathcal{U} \in \mathbb{R}^{2 \times 2 \times 2}$ with CP rank of $1$. Assume that the entries $(1,1,1),$ $(2,1,1),$ $(1,2,1)$ and $(1,1,2)$ are observed. As a result of Lemma \ref{rankunfsmCP} in this paper, all unfoldings of this tensor are rank-$1$ matrices. It is shown in Section II of \cite{ashraphijuo} that having any $4$ entries of a rank-$1$  $2 \times 4$ matrix, there are infinitely many completions for it. As a result, any unfolding of $\mathcal{U}$ is infinitely many completable given only the corresponding rank constraint. Next, using the CP decomposition \eqref{CPdecom}, we show that there are only finitely many completions of the sampled tensor of CP rank $1$.

Define $\mathbf{a}_{1}^{1}=[x \  x^{\prime}]^{\top} \in \mathbb{R}^{2}$, $\mathbf{a}_{2}^{1}=[y \  y^{\prime}]^{\top} \in \mathbb{R}^{2 }$ and $\mathbf{a}_{3}^{1}=[z \  z^{\prime}]^{\top} \in \mathbb{R}^{2 }$. Then, according to  \eqref{CPdecom}, we have the following
\begin{align}
\mathcal{U}(1,1,1) &= xyz, & \mathcal{U}(2,2,1) &= x^{\prime}y^{\prime}z, \\ \nonumber
\mathcal{U}(2,1,1) &= x^{\prime}yz,  & \mathcal{U}(2,1,2) &= x^{\prime}yz^{\prime} , \\ \nonumber
\mathcal{U}(1,2,1) &= xy^{\prime}z, & \mathcal{U}(1,2,2) &= xy^{\prime}z^{\prime}, \\ \nonumber
\mathcal{U}(1,1,2) &= xyz^{\prime},  & \mathcal{U}(2,2,2) &= x^{\prime}y^{\prime}z^{\prime}. \nonumber
\end{align}

Recall that $(1,1,1),(2,1,1),(1,2,1),$ and $(1,1,2)$ are the observed entries. Hence, the unknown entries can be determined uniquely in terms of the $4$ observed entries as
\begin{eqnarray}
\mathcal{U}(2,2,1) &=& x^{\prime}y^{\prime}z = \frac{\mathcal{U}(2,1,1)\mathcal{U}(1,2,1)}{\mathcal{U}(1,1,1)}, \\ \nonumber
\mathcal{U}(2,1,2) &=& x^{\prime}yz^{\prime} = \frac{\mathcal{U}(2,1,1)\mathcal{U}(1,1,2)}{\mathcal{U}(1,1,1)}, \\ \nonumber
\mathcal{U}(1,2,2) &=& xy^{\prime}z^{\prime} = \frac{\mathcal{U}(1,2,1)\mathcal{U}(1,1,2)}{\mathcal{U}(1,1,1)}, \\ \nonumber
\mathcal{U}(2,2,2) &=& x^{\prime}y^{\prime}z^{\prime} = \frac{\mathcal{U}(2,1,1)\mathcal{U}(1,2,1)\mathcal{U}(1,1,2)}{\mathcal{U}(1,1,1)\mathcal{U}(1,1,1)}. \nonumber
\end{eqnarray}

Therefore, based on the CP decomposition, the sampled tensor $\mathcal{U}$ is finitely (uniquely) many completable. Hence, this example illustrates that collapsing a tensor into a matrix results in loss of information and thus motivate the investigation of the tensor CP manifold.

\section{Deterministic Conditions for Finite Completability}\label{sec2}

In this section, we characterize the necessary and sufficient condition on the sampling pattern for finite completability of the sampled tensor given its CP rank. In Section \ref{subsecgeo}, we define a polynomial based on each observed entry and through studying the geometry of the manifold of the corresponding CP rank, we transform the problem of finite completability of $\mathcal{U}$ to the problem of including enough number of algebraically independent polynomials among the defined polynomials for the observed entries. In Section \ref{subscoten}, a binary tensor is constructed based on the sampling pattern $\Omega$, which allows us to study the algebraic independence of a subset of polynomials among all defined polynomials based on the samples. In Section \ref{subsecalgind}, we characterize the connection between the maximum number of algebraically independent polynomials among all the defined polynomials and finite completability of the sampled tensor.

\subsection{Geometry}\label{subsecgeo}

Suppose that the sampled tensor $\mathcal{U}$ is chosen generically from the manifold of the tensors in $\mathbb{R}^{n_1  \times \cdots \times n_d}$ of rank $r$. Assume that $\mathbf{a}_{i}^{l}$ vectors are unknown for $1 \leq i \leq d$ and $1 \leq l \leq r$. For notational simplicity, define the tuples $\mathbb{A}_l =(\mathbf{a}_{1}^{l},\mathbf{a}_{2}^{l},\dots,\mathbf{a}_{d}^{l})$ for $l=1,\dots,r$ and $\mathbb{A}=(\mathbb{A}_1,\dots,\mathbb{A}_r)$. Moreover, define $\mathbf{A}_i = [\mathbf{a}_{i}^{1}|\mathbf{a}_{i}^{2}|\dots|\mathbf{a}_{i}^{r}] \in \mathbb{R}^{n_i \times r}$. Note that each of the sampled entries results in a polynomials in terms of the entries of $\mathbb{A}$ as in \eqref{CPdecompolyeq}.

Here, we briefly mention the following two facts to highlight the fundamentals of our proposed analysis.
 
\begin{itemize}
\item {\bf Fact $1$}: As it can be seen from  \eqref{CPdecompolyeq}, any observed entry $\mathcal{U}(\vec{x})$ results in an equation that involves one entry of $\mathbf{a}_{i}^{l}$, $i=1,\dots,d$ and $l = 1,\dots,r$. Considering the entries of $\mathbb{A}$ as variables (right-hand side of  \eqref{CPdecompolyeq}), each observed entry results in a polynomial in terms of these variables. Moreover, for any observed entry $\mathcal{U}(\vec{x})$, the value of $x_i$ specifies the location of the entry of $\mathbf{a}_{i}^{l}$ that is involved in the corresponding polynomial, $i=1,\dots,d$ and $l = 1,\dots,r$.

\item {\bf Fact $2$}: It can be concluded from Bernstein's theorem \cite{Bernstein} that in a system of $n$ polynomials in $n$ variables with coefficients chosen generically, the polynomials are algebraically independent with probability one, and therefore there exist only finitely many solutions. Moreover, in a system of $n$ polynomials in $n-1$ variables (or less), polynomials are algebraically dependent with probability one.

\end{itemize}

The following assumption will be used frequently in this paper.

{\bf Assumption $1$}: Each row of the $d$-th matricization of the sampled tensor, i.e., $\mathbf{U}_{(d)}$ includes at least $r$ observed entries.

\begin{lemma}\label{assum}
Given $\mathbf{A}_i$'s for $i=1,\dots,d-1$ and Assumption $1$, $\mathbf{A}_d$ can be determined uniquely.
\end{lemma}

\begin{proof}
Each row of $\mathbf{A}_d$ has $r$ entries and also as it can be seen from \eqref{CPdecompolyeq}, each observed entry in the $i$-th row of $\mathbf{U}_{(d)}$ results in a degree-$1$ polynomial in terms of the $r$ entries of the $i$-th row of $\mathbf{A}_d$. Since Assumption $1$ holds, for each row of $\mathbf{A}_d$ that has $r$ variables, we have at least $r$ degree-$1$ polynomials. Genericity of the coefficients of these polynomials results that each row of $\mathbf{A}_d$ can be determined uniquely.
\end{proof}

As a result of Lemma \ref{assum}, we can obtain $\mathbf{A}_d$ in terms of the entries of $\mathbf{A}_i$'s for $i=1,\dots,d-1$. As mentioned earlier, each observed entry is equivalent to a polynomial in the format of \eqref{CPdecompolyeq}. Consider all such polynomials excluding those that have been used to obtain $\mathbf{A}_d$ ($r$ samples at each row of $\mathbf{U}_{(d)}$) and denote this set of polynomials in terms of the entries of $\mathbf{A}_i$'s for $i=1,\dots,d-1$ by $\mathcal{P}(\Omega)$.

We are interested in defining an equivalence class such that each class includes exactly one of the decompositions among all rank-$r$ decompositions of a particular tensor and the pattern in Lemma \ref{patternCP} characterizes such an equivalence class. Lemma \ref{basel0} below is a re-statement of Lemma $ 1$ in \cite{ashraphijuo3}, which characterizes such an equivalence class or equivalently geometric pattern for a matrix instead of tensor. This lemma will be used to show Lemma \ref{patternCP} later.


\begin{lemma}\label{basel0}
Let $\mathbf{X}$ denote a generically chosen matrix from the manifold of $n_1 \times n_2$ matrices of rank $r$ and also $\mathbf{Q} \in \mathbb{R}^{r \times r}$ be a given full rank matrix. Then, there exists a unique decomposition $\mathbf{X}= \mathbf{YZ}$ such that $\mathbf{Y} \in \mathbb{R}^{n_1 \times r}$, $\mathbf{Z} \in \mathbb{R}^{r \times n_2}$ and $\mathbf{P} = \mathbf{Q}$, where $\mathbf{P} \in \mathbb{R}^{r \times r}$ represents a submatrix{\footnote{Specified by a subset of rows and a subset of columns (not necessarily consecutive).}} of $\mathbf{Y}$.
\end{lemma}

In Lemma \ref{patternCP}, we generalize Lemma \ref{basel0} and characterize the similar pattern for a multi-way tensor. Assuming that $\mathbf{P}$ represents the submatrix of $\mathbf{Y}$ consists of the first $r$ columns and the first $r$ rows of $\mathbf{Y}$ and also $\mathbf{Q}$ is equal to the $r \times r$ identity matrix, this pattern is called the canonical decomposition of $\mathbf{X}$. The canonical decomposition is shown for a rank-$2$ matrix as the following
\begin{center}
\begin{tabular}{ |c|c|c|c| } 
 \hline
$1$ & $-1$ & $ \ 0 \ $ & $-1$ \\ \hline
$2$ & $2$ & $4$ & $6$ \\ \hline
$-1$ & $3$ & $2$ & $5$ \\ \hline
$1$ & $2$ & $3$ & $5$ \\ 
 \hline
\end{tabular} 
 \  \ \ $=$ \ \ \ 
\begin{tabular}{ |c|c| } 
 \hline
$1$ &  $0$ \\ \hline
$0$ & $1$ \\ \hline
$y_1$ & $y_2$ \\ \hline
$y_3$ & $y_4$  \\ 
 \hline
\end{tabular}
 \  \ \ $\times$ \ \ \ 
 \begin{tabular}{ |c|c|c|c| } 
 \hline
$x_1$ & $x_2$ & $x_3$ & $x_4$ \\ \hline
$x_5$ & $x_6$ & $x_7$ & $x_8$ \\ 
 \hline
\end{tabular} \ ,
\end{center}
where $x_i$'s and $y_i$'s can be determined uniquely as 
\begin{center}
\begin{tabular}{ |c|c| } 
 \hline
$y_1$ & $y_2$ \\ \hline
$y_3$ & $y_4$  \\ 
 \hline
\end{tabular}
 \  \  $=$ \ \  
\begin{tabular}{ |c|c| } 
 \hline
$-2$ & $ \ \frac{1}{2} \ $ \\ \hline
$- \frac{1}{2}$ & $\frac{3}{4}$  \\ 
 \hline
\end{tabular}
 \ \ and \ \
  \begin{tabular}{ |c|c|c|c| } 
 \hline
$x_1$ & $x_2$ & $x_3$ & $x_4$ \\ \hline
$x_5$ & $x_6$ & $x_7$ & $x_8$ \\ 
 \hline
\end{tabular}
 \  \  $=$ \ \  
 \begin{tabular}{ |c|c|c|c| } 
 \hline
$\ 1 \ $ & $-1$ & $ \ 0 \ $ & $-1$ \\ \hline
$2$ & $2$ & $4$ & $6$ \\ 
 \hline
\end{tabular} \ .
\end{center}

Also, the above canonical decomposition can be written as the following
\begin{center}
\begin{tabular}{ |c|c|c|c| } 
 \hline
$1$ & $-1$ & $ \ 0 \ $ & $-1$ \\ \hline
$2$ & $2$ & $4$ & $6$ \\ \hline
$-1$ & $3$ & $2$ & $5$ \\ \hline
$1$ & $2$ & $3$ & $5$ \\ 
 \hline
\end{tabular} 
 \ \ $=$ \ \ 
 \begin{tabular}{ |c| } 
 \hline
$1$  \\ \hline
$0$  \\ \hline
$y_1$  \\ \hline
$y_3$   \\ 
 \hline
\end{tabular}
 \  \  $\times$ \ \  
 \begin{tabular}{ |c|c|c|c| } 
 \hline
$x_1$ & $x_2$ & $x_3$ & $x_4$ \\
 \hline
\end{tabular} 
 \ \  + \ \  
\begin{tabular}{ |c| } 
 \hline
 $0$ \\ \hline
$1$ \\ \hline
 $y_2$ \\ \hline
 $y_4$  \\ 
 \hline
\end{tabular}
 \  \ $\times$ \ \ 
 \begin{tabular}{ |c|c|c|c| } 
 \hline
$x_5$ & $x_6$ & $x_7$ & $x_8$ \\ 
 \hline
\end{tabular} \ .
\end{center}

Generalization of the canonical decomposition for multi-way tensor  is as the following
\begin{center}
$\mathbf{a}_1^1 \ = \ $
\begin{tabular}{ |c| } 
 \hline
 $1$ \\ \hline
  $0$ \\ \hline
 $\vdots$ \\ \hline
$0$ \\ \hline
 $\mathbf{a}_1^1(r+1)$ \\ \hline
$\vdots$ \\ \hline
 $\mathbf{a}_1^1(n_1)$  \\ 
 \hline
\end{tabular}
 $, \ \dots \ , \ $
 $\mathbf{a}_1^r \ = \ $
\begin{tabular}{ |c| } 
 \hline
 $0$ \\ \hline
  $0$ \\ \hline
 $\vdots$ \\ \hline
$1$ \\ \hline
 $\mathbf{a}_1^r(r+1)$ \\ \hline
$\vdots$ \\ \hline
 $\mathbf{a}_1^r(n_1)$  \\ 
 \hline
\end{tabular} \ ,
\end{center}
and for $i \in \{2,\dots,d-1\}$
\begin{center}
$\mathbf{a}_i^1 \ = \ $
\begin{tabular}{ |c| } 
 \hline
 $1$ \\ \hline
 $\mathbf{a}_i^1(2)$ \\ \hline
$\vdots$ \\ \hline
 $\mathbf{a}_i^1(n_i)$  \\ 
 \hline
\end{tabular}
 $, \ \dots \ , \ $
 $\mathbf{a}_i^r \ = \ $
\begin{tabular}{ |c| } 
 \hline
$1$ \\ \hline
 $\mathbf{a}_i^r(2)$ \\ \hline
$\vdots$ \\ \hline
 $\mathbf{a}_i^r(n_i)$  \\ 
 \hline
\end{tabular} \ .
\end{center}

\begin{lemma}\label{patternCP}
Let $j \in \{1,\dots,d-1\}$ be a fixed number and define $J= \{1,\dots,d-1\} \backslash \{j\}$. Assume that the full rank matrix $\mathbf{Q}_j \in \mathbb{R}^{r \times r}$ and matrices $\mathbf{Q}_i \in \mathbb{R}^{1 \times r}$ with nonzero entries for $i \in J$ are given. Also, let $\mathbf{P}_i $ denote an arbitrary submatrix of $\mathbf{A}_i$, $i = 1, 2, \dots , d-1$, where $\mathbf{P}_j \in \mathbb{R}^{r \times r}$ and $\mathbf{P}_i \in \mathbb{R}^{1 \times r}$ for $i \in J$. Then, with probability one, there exists exactly one rank-$r$ decomposition of $\mathcal{U}$ such that $\mathbf{P}_i=\mathbf{Q}_i$, $i=1,\dots,d-1$.
\end{lemma}

\begin{proof}
First we claim that there exists at most one rank-$r$ decomposition of $\mathcal{U}$ such that $\mathbf{P}_i=\mathbf{Q}_i$, $i=1,\dots,d-1$. We assume that $\mathbf{P}_i=\mathbf{Q}_i$, $i=1,\dots,d-1$ and also $\mathcal{U}$ is given. Then, it suffices to show that the rest of the entries of $\mathbb{A}$ can be determined in at most a unique way (no more than one solution) in terms of the given parameters such that \eqref{CPdecom} holds. Note that if a variable can be determined uniquely through two different ways (two sets of samples or equations), in general either it can be determined uniquely if both ways result in the same value or it does not have any solution otherwise. Let $y_i$ denote the row number of submatrix $\mathbf{P}_i \in \mathbb{R}^{1 \times r}$ for $i \in J$ and $Y_j = \{y_j^1,\dots,y_j^r\}$ denote the row numbers of submatrix $\mathbf{P}_j \in \mathbb{R}^{r \times r}$.

As the first step of proving our claim, we show that $\mathbf{A}_d$ can be determined uniquely. Consider the subtensor $\mathcal{U}^{\prime} = \mathcal{U}(y_1,\dots,y_{j-1},Y_j,y_{j+1},\dots,y_{d-1},:) \in \mathbb{R}^{ \overbrace{ 1\times \dots \times 1}^{j-1} \times \textit{\large r} \times \overbrace{1 \times \dots \times 1}^{d-j-1} \times \textit{\large n}_d}$ which includes $rn_d$ entries. Having CP decomposition \eqref{CPdecompolyeq}, each entry of $\mathcal{U}^{\prime}$ results in one degree-$1$ polynomial in terms of the entries of $\mathbf{A}_d$ with coefficients in terms of the entries of $\mathbf{Q}_i$'s. Let the matrix $\mathbf{U}^{\prime} \in \mathbb{R}^{r \times n_d}$ represent the $r n_d$ entries of $\mathcal{U}^{\prime}$. Moreover, define $\mathbf{C}= [\mathbf{c}_1|\dots | \mathbf{c}_r] \in \mathbb{R}^{r \times r}$ where $\mathbf{c}_l = \left( \Pi_{i \in J} \mathbf{P}_i(1,l) \right) \mathbf{q}_j^{l} \in \mathbb{R}^{r \times 1}$ for $l=1,\dots,r$ and $\mathbf{q}_j^{l} \in \mathbb{R}^{r \times 1}$ is the $l$-th column of $\mathbf{Q}_j$.

Observe that CP decomposition \eqref{CPdecompolyeq} for the entries of $\mathcal{U}^{\prime}$ can be written as $ \mathbf{U}^{\prime} = \mathbf{C} \mathbf{A}_d^{\top}$. Recall that $\mathbf{Q}_j$ is full rank, and therefore $\mathbf{q}_j^{l}$'s are linearly independent, $l=1,\dots,r$. Also, a system of equations with at least $m$ linearly independent degree-$1$ polynomials in $m$ variables does not have more than one solution. Hence, $\mathbf{c}_l$'s are also linearly independent for $l=1,\dots,r$, and therefore $\mathbf{C}$ is full rank. As a result, $\mathbf{A}_d$ can be determined uniquely. In the second step, similar to the first step, we can show that the rest of $\mathbf{A}_i$'s have at most one solution having one entry of $\mathbf{A}_d$ which has been already obtained.

Finally, we also claim that there exists at least one rank-$r$ decomposition of $\mathcal{U}$ such that $\mathbf{P}_i=\mathbf{Q}_i$, $i=1,\dots,d-1$. We show this by induction on $d$. For $d=2$, this is a result of Lemma \ref{basel0}. Induction hypothesis states that the claim holds for $d=k-1$ and we need to show that it also holds for $d=k$. Since by merging dimension $k-1$ and $k$ for each of the rank-$1$ tensors of the corresponding CP decomposition and using induction hypothesis this step reduces to showing a rank-$1$ matrix can be decomposed to two vectors such that one component of one of them is given which is again a special case of Lemma \ref{basel0} for rank-$1$ scenario.
\end{proof}

Assume that $\mathcal{S}$ denotes the set of all possible $\mathbf{A}_i$'s for $i=1,\dots,d-1$ given $\mathbf{A}_d$ without any polynomial constraint. Lemma \ref{patternCP} results in a pattern that characterizes exactly one rank-$r$ decomposition among all rank-$r$ decompositions, and therefore the dimension of $\mathcal{S}$ is equal to the number of unknowns, i.e., number of entries of $\mathbf{A}_i$'s for $i=1,\dots,d-1$ excluding those that are involved in the pattern $\mathbf{P}_i$'s in Lemma \ref{patternCP} which is $r(\sum_{i=1}^{d-1} n_i) - r^2 - r(d-2) $.

\begin{lemma}\label{thmnumpolyind}
For almost every $\mathcal{U}$, the sampled tensor is finitely completable if and only if the maximum number of algebraically independent polynomials in $\mathcal{P}(\Omega)$ is equal to $r(\sum_{i=1}^{d-1} n_i) - r^2 - r(d-2) $. 
\end{lemma}

\begin{proof}
The proof is omitted due to the similarity to the proof of Lemma $2$ in \cite{ashraphijuo} with the only difference that here the dimension is $r(\sum_{i=1}^{d-1} n_i) - r^2 - r(d-2) $ instead of $  \left(\Pi_{i=1}^{j} n_i\right) \left(  \Pi_{i=j+1}^{d} r_i\right)  - \left( \sum_{i=j+1}^{d}   r_i^{2} \right) $ which is the dimension of the core in Tucker decomposition.
\end{proof}

%

\subsection{Constraint Tensor}\label{subscoten}

In this section, we provide a procedure to construct a binary tensor ${\breve{\Omega}}$ based on $\Omega$ such that $\mathcal{P}({\breve{\Omega}}) = \mathcal{P}(\Omega) $ and each polynomial can be represented by one $d$-way subtensor of ${\breve{\Omega}}$ which belongs to $\mathbb{R}^{n_1 \times n_2 \times \cdots \times n_{d-1} \times 1 }$. Using ${\breve{\Omega}}$, we are able to recognize the observed entries that have been used to obtain the $\mathbf{A}_{d}$ in terms of the entries of $\mathbf{A}_{1},\dots ,\mathbf{A}_{d-1}$, and we can study the algebraic independence of the polynomials in $\mathcal{P}(\Omega) $ which is directly related to finite completability through Lemma \ref{thmnumpolyind}.

For each subtensor $\mathcal{Y}$ of the sampled tensor $\mathcal{U}$, let $N_{\Omega}(\mathcal{Y})$ denote the number of sampled entries in $\mathcal{Y}$. Specifically, consider any subtensor $\mathcal{Y} \in \mathbb{R}^{n_1 \times n_2 \times \cdots \times n_{d-1} \times 1 }$ of the tensor $\mathcal{U}$. Then, since $r$ of the polynomials have been used to obtain $\mathbf{A}_d$, $\mathcal{Y}$ contributes $N_{\Omega}(\mathcal{Y}) - r$ polynomial equations in terms of the entries of $\mathbf{A}_{1},\dots ,\mathbf{A}_{d-1}$ among all $N_{\Omega}(\mathcal{U}) -r n_d$ polynomials in $\mathcal{P}(\Omega)$.

The sampled tensor $\mathcal{U}$ includes $n_d$ subtensors that belong to  $\mathbb{R}^{n_1 \times n_2 \times \cdots \times n_{d-1} \times 1 }$ and let $\mathcal{Y}_i$ for $1 \leq i \leq n_d$ denote these $n_d$ subtensors. Define a binary valued tensor $\mathcal{\breve{Y}}_{i} \in \mathbb{R}^{n_1 \times n_2 \times \cdots \times n_{d-1} \times  k_i}$, where $k_i= N_{\Omega}(\mathcal{Y}_{i}) - r$ and its entries are described as the following. We can look at $\mathcal{\breve{Y}}_{i}$ as $k_i$ tensors each belongs to $\mathbb{R}^{n_1 \times n_2 \times \cdots \times n_{d-1} \times 1 }$. For each of the mentioned $k_i$ tensors in $\mathcal{\breve{Y}}_{i}$ we set the entries corresponding to the $r$ observed entries that are used to obtain $\mathbf{A}_{d}$  equal to $1$. For each of the other $k_i$ observed entries that have not been used to obtain $\mathbf{A}_{d}$, we pick one of the $k_i$ tensors of $\mathcal{\breve{Y}}_{i}$ and set its corresponding entry (the same location as that specific observed entry) equal to $1$ and set the rest of the entries equal to $0$. In the case that $k_i=0$ we simply ignore $\mathcal{\breve{Y}}_{i}$, i.e., $\mathcal{\breve{Y}}_{i} = \emptyset$

By putting together all $n_d$ tensors in dimension $d$, we construct a binary valued tensor ${\breve{\Omega}} \in \mathbb{R}^{n_1 \times n_2 \times \cdots \times n_{d-1} \times K}$, where $K = \sum_{i=1}^{n_d} k_i = N_{\Omega}(\mathcal{U}) - r n_d$ and call it the {\bf constraint tensor}. Observe that each subtensor of ${\breve{\Omega}}$ which belongs to $\mathbb{R}^{n_1 \times n_2 \times \cdots \times n_{d-1} \times 1 }$ includes exactly $r+1$ nonzero entries. In the following we show this procedure for an example.

\begin{example}
{\rm Consider an example in which $d=3$ and $r=2$ and $\mathcal{U} \in \mathbb{R}^{3 \times 3 \times 3}$. Assume that $\Omega(x,y,z)=1$ if $(x,y,z) \in \mathcal{S}$ and $\Omega(x,y,z)=0$ otherwise, where 
\begin{eqnarray}
\mathcal{S} = \{(1,1,1),(1,2,1),(2,3,1), (3,3,1),(1,1,2),(2,1,2),(3,2,2),(1,3,3),(3,2,3)\}, \nonumber
\end{eqnarray}
represents the set of observed entries. Hence, observed entries $(1,1,1),(1,2,1),(2,3,1), (3,3,1)$ belong to $\mathcal{Y}_{1}$, observed entries $(1,1,2),(2,1,2),(3,2,2)$ belong to $\mathcal{Y}_{2}$, and observed entries $(1,3,3),(3,2,3)$ belong to $\mathcal{Y}_{3}$. As a result, $k_1 = 4-2 =2$, $k_2 =3-2=1$, and $k_3 =2-2=0$. Hence, $\mathcal{\breve{Y}}_{1} \in \mathbb{R}^{3 \times 3 \times 2 }$, $\mathcal{\breve{Y}}_{2} \in \mathbb{R}^{3 \times 3 \times 1 }$, and $\mathcal{\breve{Y}}_{3} = \emptyset$, and therefore the constraint tensor ${\breve{\Omega}}$ belongs to $\mathbb{R}^{3 \times 3 \times 3}$.

Also, assume that the entries that we use to obtain $\mathbf{A}_{3}$ in terms of the entries of $\mathbf{A}_{1}$ and $\mathbf{A}_{2}$ are $(2,3,1)$, $(3,3,1)$, $(1,1,2)$, $(2,1,2)$, $(1,3,3)$ and $(3,2,3)$. Note that $\mathcal{\breve{Y}}_{1}(2,3,1) = \mathcal{\breve{Y}}_{1}(2,3,2) = \mathcal{\breve{Y}}_{1}(3,3,1) = \mathcal{\breve{Y}}_{1}(3,3,2) =1$ (correspond to entries of $\mathcal{Y}_1$ that have been used to obtain $\mathbf{A}_3$), and also for the two other observed entries we have $\mathcal{\breve{Y}}_{1}(1,1,1) =1 $ (correspond to $\mathcal{U}(1,1,1)$) and $\mathcal{\breve{Y}}_{1}(1,2,2)=1$ (correspond to $\mathcal{U}(1,2,1)$) and the rest of the entries of $\mathcal{\breve{Y}}_{1}$ are equal to zero. Similarly, $\mathcal{\breve{Y}}_{2}(1,1,1) = \mathcal{\breve{Y}}_{2}(2,1,1) = \mathcal{\breve{Y}}_{2}(3,2,1)=1$ and the rest of the entries of $\mathcal{\breve{Y}}_{2}$ are equal to zero.

Then, ${\breve{\Omega}}(x,y,z)=1$ if $(x,y,z) \in \mathcal{S}^{\prime}$ and ${\breve{\Omega}}(x,y,z)=0$ otherwise, where 
\begin{eqnarray}
\mathcal{\breve{S}} = \{(1,1,1),(1,2,2),(2,3,1),(2,3,2),(3,3,1),(3,3,2),(1,1,3),(2,1,3),(3,2,3)\}. \nonumber
\end{eqnarray}  } {  \hfill  \qedsymbol}
\end{example}

Note that each subtensor of ${\breve{\Omega}}$ that belongs to $\mathbb{R}^{n_1 \times \dots \times n_{d-1} \times 1}$ represents one of the polynomials in $\mathcal{P}(\Omega)$ besides showing the polynomials that have been used to obtain $\mathbf{A}_{d}$. More specifically, consider a subtensor of ${\breve{\Omega}}$ that belongs to $\mathbb{R}^{n_1 \times \dots \times n_{d-1} \times 1}$ with $r+1$ nonzero entries. Observe that exactly $r$ of them correspond to the observed entries that have been used to obtain $\mathbf{A}_{d}$. Hence, this subtensor represents a polynomial after replacing entries of $\mathbf{A}_{d}$ by the expressions in terms of entries of $\mathbf{A}_{1},\dots ,\mathbf{A}_{d-1}$, i.e., a polynomial in $\mathcal{P}(\Omega)$.

\subsection{Algebraic Independence}\label{subsecalgind}

In this section, we obtain the maximum number of algebraically independent polynomials in $\mathcal{P}({\breve{\Omega}})$ in terms of the simple geometrical structure of nonzero entries of $\Omega$, i.e., the locations of the sampled entries. On the other hand, Lemma \ref{thmnumpolyind} provides the required number of algebraically independent polynomials in $\mathcal{P}(\Omega)$ for finite completability. Hence, at the end of this section, we obtain the necessary and sufficient deterministic conditions on the sampling pattern for finite completability.

According to Lemma \ref{patternCP}, as we consider one particular equivalence class some of the entries of $\mathbf{A}_{i}$'s are known, i.e., $\mathbf{P}_{1},\dots,\mathbf{P}_{d-1}$ in the statement of the lemma. Therefore, in order to find the number of variables (unknown entries of $\mathbf{A}_{i}$'s) in a set of polynomials, we should subtract the number of known entries in the corresponding pattern from the total number of involved entries. Also, recall that the sampled tensor is chosen generically from the corresponding manifold, and therefore according to Fact $2$, the independency of the polynomials can be studied through studying the number of variables involved in each subset of them.

\begin{definition}
Let ${\breve{\Omega}}^{\prime} \in \mathbb{R}^{n_1 \times n_2 \times \cdots \times n_{d-1} \times t}$ be a subtensor of the constraint tensor ${\breve{\Omega}}$. Let $m_i({\breve{\Omega}}^{\prime})$ denote the number of nonzero rows of ${\mathbf {\breve{\Omega}}}^{\prime}_{(i)}$. Also, let $\mathcal{P}({\breve{\Omega}}^{\prime})$ denote the set of polynomials that correspond to nonzero entries of ${\breve{\Omega}}^{\prime}$.
\end{definition}

The following lemma gives an upper bound on the maximum number of algebraically independent polynomials in the set $\mathcal{P}({\breve{\Omega}}^{\prime})$ for an arbitrary subtensor ${\breve{\Omega}}^{\prime} \in \mathbb{R}^{n_1 \times n_2 \times \cdots \times n_{d-1} \times t}$ of the constraint tensor. Note that $\mathcal{P}({\breve{\Omega}}^{\prime})$ includes exactly $t$ polynomials as each subtensor belonging to $\mathbb{R}^{n_1 \times n_2 \times \cdots \times n_{d-1} \times 1}$ represents one polynomial.

\begin{lemma}\label{uppboundlemma}
Suppose that Assumption $1$ holds. Consider an arbitrary subtensor ${\breve{\Omega}}^{\prime} \in \mathbb{R}^{n_1 \times n_2 \times \cdots \times n_{d-1} \times t}$ of the constraint tensor ${\breve{\Omega}}$. The maximum number of algebraically independent polynomials in $\mathcal{P}({\breve{\Omega}}^{\prime})$ is at most
\begin{eqnarray}\label{uppboundindp}
r \left( \left( \sum_{i=1}^{d-1} m_i (\breve{\Omega}^{\prime}) \right) -   \text{min} \left\{ \text{max} \left\{ m_1 (\breve{\Omega}^{\prime}) , \dots ,m_{d-1}(\breve{\Omega}^{\prime}) \right\} ,r \right\} - (d-2) \right).
\end{eqnarray}
\end{lemma}

\begin{proof}
As a consequence of Fact $2$, the maximum number of algebraically independent polynomials in a subset of polynomials of $\mathcal{P}({\breve{\Omega}}^{\prime})$ is at most equal to the total number of variables that are involved in the corresponding polynomials. Note that by observing the structure of \eqref{CPdecompolyeq} and Fact $1$, the number of entries of $\mathbf{A}_i$ that are involved in the polynomials $\mathcal{P}({\breve{\Omega}}^{\prime})$ is equal to $r m_i (\breve{\Omega}^{\prime})$, $i=1,\dots,d-1$. Therefore, the total number of entries of $\mathbf{A}_1, \dots , \mathbf{A}_{d-1}$ that are involved in the polynomials $\mathcal{P}({\breve{\Omega}}^{\prime})$ is equal to $r  \left( \sum_{i=1}^{d-1} m_i (\breve{\Omega}^{\prime}) \right)$. However, some of the entries of $\mathbf{A}_1, \dots , \mathbf{A}_{d-1}$ are known and depending on the equivalence class we should subtract them from the total number of involved entries.

For a fixed number $j$ in Lemma \ref{patternCP}, it is easily verified that the total number of variables (unknown entries) of $\mathbf{A}_1, \dots , \mathbf{A}_{d-1}$ that are involved in the polynomials $\mathcal{P}({\breve{\Omega}}^{\prime})$ is equal to $r \sum_{i \in J} \left( m_i (\breve{\Omega}^{\prime}) - 1  \right)^+ + r \left( m_j (\breve{\Omega}^{\prime}) - r  \right)^+ $, with $J= \{1,\dots,d-1\} \backslash \{j\} $. Note that $\left( m_i (\breve{\Omega}^{\prime}) - 1  \right)^+ = m_i (\breve{\Omega}^{\prime}) - 1$, $\left( m_j (\breve{\Omega}^{\prime}) - r  \right)^+ =  m_j (\breve{\Omega}^{\prime}) - \text{min} \left\{ m_{j}(\breve{\Omega}^{\prime}) ,r \right\}$. However, $j$ is not a fixed number in general. Therefore, the maximum number of known entries of $\mathbf{A}_1, \dots , \mathbf{A}_{d-1}$ that are involved in the polynomials $\mathcal{P}({\breve{\Omega}}^{\prime})$ is equal to \\ $r \left( \text{min} \left\{ \text{max} \left\{ m_1 (\breve{\Omega}^{\prime}) , \dots ,m_{d-1}(\breve{\Omega}^{\prime}) \right\} ,r \right\} - (d-2) \right) $, which results that the number of variables of $\mathbf{A}_1, \dots , \mathbf{A}_{d-1}$ that are involved in the polynomials $\mathcal{P}({\breve{\Omega}}^{\prime})$ is equal to \eqref{uppboundindp}.
\end{proof}

The set of polynomials corresponding to ${\breve{\Omega}}^{\prime}$, i.e., $\mathcal{P}({\breve{\Omega}}^{\prime})$ is called minimally algebraically dependent if the polynomials in $\mathcal{P}({\breve{\Omega}}^{\prime})$ are algebraically dependent but polynomials in every of its proper subsets are algebraically independent. The following lemma which is Lemma $3$ in \cite{ashraphijuo}, provides an important property about a set of minimally algebraically dependent $\mathcal{P}({\breve{\Omega}}^{\prime})$. This lemma will be used later to derive the maximum number of algebraically independent polynomials in $\mathcal{P}({\breve{\Omega}}^{\prime})$.

\begin{lemma}\label{minimdeppolylemma}
Suppose that Assumption $1$ holds. Consider an arbitrary subtensor ${\breve{\Omega}}^{\prime} \in \mathbb{R}^{n_1 \times n_2 \times \cdots \times n_{d-1} \times t}$ of the constraint tensor ${\breve{\Omega}}$. Assume that polynomials in $\mathcal{P}({\breve{\Omega}}^{\prime})$ are minimally algebraically dependent. Then, the number of variables (unknown entries) of $\mathbf{A}_1, \dots , \mathbf{A}_{d-1}$ that are involved in $\mathcal{P}({\breve{\Omega}}^{\prime})$ is equal to $t-1$.
\end{lemma}

Given an arbitrary subtensor ${\breve{\Omega}}^{\prime} \in \mathbb{R}^{n_1 \times n_2 \times \cdots \times n_{d-1} \times t}$ of the constraint tensor ${\breve{\Omega}}$, we are interested in obtaining the maximum number of algebraically independent polynomials in $\mathcal{P}({\breve{\Omega}}^{\prime})$ based on the structure of nonzero entries of ${\breve{\Omega}}^{\prime}$. The next lemma can be used to characterize this number in terms of a simple geometric structure of nonzero entries of ${\breve{\Omega}}^{\prime}$.

\begin{lemma}\label{characnumindpolylemma}
Suppose that Assumption $1$ holds. Consider an arbitrary subtensor ${\breve{\Omega}}^{\prime} \in \mathbb{R}^{n_1 \times n_2 \times \cdots \times n_{d-1} \times t}$ of the constraint tensor ${\breve{\Omega}}$. The polynomials in $\mathcal{P}({\breve{\Omega}}^{\prime})$ are algebraically independent if and only if for any $ t^{\prime} \in \{1,\dots,t\}$ and any subtensor ${\breve{\Omega}}^{\prime \prime} \in \mathbb{R}^{n_1 \times n_2 \times \cdots \times n_{d-1} \times t^{\prime}}$ of ${\breve{\Omega}}^{\prime}$ we have
\begin{eqnarray}\label{ineqindppoly}
r \left( \left( \sum_{i=1}^{d-1} m_i (\breve{\Omega}^{\prime \prime}) \right) -   \text{min} \left\{ \text{max} \left\{ m_1 (\breve{\Omega}^{\prime \prime}) , \dots ,m_{d-1}(\breve{\Omega}^{\prime \prime}) \right\} ,r \right\} - (d-2) \right) \geq t^{\prime}.
\end{eqnarray}
\end{lemma}

\begin{proof}
First, assume that all polynomials in $\mathcal{P}({\breve{\Omega}}^{\prime})$ are algebraically independent. Also, by contradiction assume that there exists a subtensor ${\breve{\Omega}}^{\prime \prime} \in \mathbb{R}^{n_1 \times n_2 \times \cdots \times n_{d-1} \times t^{\prime}}$ of ${\breve{\Omega}}^{\prime}$ that \eqref{ineqindppoly} does not hold for. Note that $\mathcal{P}({\breve{\Omega}}^{\prime \prime})$ includes $t^{\prime}$ polynomials. On the other hand, according to Lemma \ref{uppboundlemma}, the maximum number of algebraically independent polynomials in $\mathcal{P}({\breve{\Omega}}^{\prime \prime})$ is no greater than the LHS of \eqref{ineqindppoly}, and therefore the polynomials in $\mathcal{P}({\breve{\Omega}}^{\prime \prime})$ are not algebraically independent. Hence, the polynomials in $\mathcal{P}({\breve{\Omega}}^{ \prime})$ are not algebraically independent as well.

In order to prove the other side of the statement, assume that the polynomials in $\mathcal{P}({\breve{\Omega}}^{\prime})$ are algebraically dependent. Hence, there exists a subset of the polynomials that are minimally algebraically dependent and let us denote it by $\mathcal{P}({\breve{\Omega}}^{\prime \prime})$, where ${\breve{\Omega}}^{\prime \prime} \in \mathbb{R}^{n_1 \times n_2 \times \cdots \times n_{d-1} \times t^{\prime}}$ is a subtensor of ${\breve{\Omega}}^{\prime}$. As stated in Lemma \ref{minimdeppolylemma}, the number of involved variables in polynomials in $\mathcal{P}({\breve{\Omega}}^{\prime \prime})$ is equal to $t^{\prime}-1$. On the other hand, in the proof of Lemma \ref{uppboundlemma}, we showed that the number involved variables is at least equal the LHS of \eqref{ineqindppoly}. Therefore, the LHS of \eqref{ineqindppoly} is less than or equal to $t^{\prime}-1$ or equivalently 
\begin{eqnarray}\label{ineqindppolynot}
r \left( \left( \sum_{i=1}^{d-1} m_i (\breve{\Omega}^{\prime \prime}) \right) -   \text{min} \left\{ \text{max} \left\{ m_1 (\breve{\Omega}^{\prime \prime}) , \dots ,m_{d-1}(\breve{\Omega}^{\prime \prime}) \right\} ,r \right\} - (d-2) \right) < t^{\prime}.
\end{eqnarray}
\end{proof}

Finally, the following theorem characterizes the necessary and sufficient condition on $\breve{\Omega}$ for finite completability of the sampled tensor $\mathcal{U}$.

\begin{theorem}\label{mainThm}
Suppose that Assumption $1$ holds. For almost every $\mathcal{U}$, the sampled tensor $\mathcal{U}$ is finitely completable if and only if $\breve{\Omega}$ contains a subtensor ${\breve{\Omega}}^{\prime} \in \mathbb{R}^{n_1 \times n_2 \times \cdots \times n_{d-1} \times t}$ such that (i) $t = r(\sum_{i=1}^{d-1} n_i) - r^2 - r(d-2)$ and (ii) for any $ t^{\prime} \in \{1,\dots,t\}$ and any subtensor ${\breve{\Omega}}^{\prime \prime} \in \mathbb{R}^{n_1 \times n_2 \times \cdots \times n_{d-1} \times t^{\prime}}$ of ${\breve{\Omega}}^{\prime}$,  \eqref{ineqindppoly} holds.
\end{theorem}

\begin{proof}
Lemma \ref{thmnumpolyind} states that for almost every $\mathcal{U}$, there exist finitely many completions of the sampled tensor if and only if $\mathcal{P}(\breve{\Omega})$ includes $r(\sum_{i=1}^{d-1} n_i) - r^2 - r(d-2)$ algebraically independent polynomials. Moreover, according to Lemma \ref{characnumindpolylemma}, polynomials corresponding to a subtensor ${\breve{\Omega}}^{\prime} \in \mathbb{R}^{n_1 \times n_2 \times \cdots \times n_{d-1} \times t}$ of the constraint tensor are algebraically independent if and only if condition (ii) of the statement of the Theorem holds. Therefore, for almost every $\mathcal{U}$, conditions (i) and (ii) hold if and only if the sampled tensor $\mathcal{U}$ is finitely completable.
\end{proof}

\section{Probabilistic Conditions for Finite Completability}\label{sec3}

Assume that the entries of the tensor are sampled independently with probability $p$. In this section, we are interested in obtaining a condition in terms of the number of samples, i.e., the sampling probability, to ensure the finite completability of the sampled tensor with high probability. In Section \ref{subsunf}, we apply the existing results on the Grassmannian manifold in \cite{charact} on any of the unfoldings of the sampled tensor to derive the mentioned probabilistic condition. In Section \ref{subsCP}, we obtain the conditions on the number of samples to ensure that conditions (i) and (ii) in the statement of Theorem \ref{mainThm} hold with high probability or in other words, to ensure the finite completability with high probability. For the notational simplicity in this section, we assume that $n_1=n_2=\dots = n_d$, i.e., $\mathcal{U} \in \mathbb{R}^{n \times n \times \dots \times n}$.

\subsection{Unfolding Approach}\label{subsunf}

In this section, we are interested in applying the existing analysis based on the Grassmannian manifold to obtain probabilistic conditions on the sampling pattern for finite completability with high probability. The following theorem restates Theorem $3$ in \cite{charact}.

\begin{theorem}\label{Grassmannianprob}
Consider an $n \times N$ matrix with the given rank $k$ and let $0 < \epsilon < 1$ be given. Suppose $k \leq \frac{n}{6}$ and that each {\bf column} of the sampled matrix is observed in at least $l$ entries, distributed uniformly at random and independently across entries, where
\begin{eqnarray}\label{matrixprobineq}
l > \max\left\{12 \ \log \left( \frac{n}{\epsilon} \right) + 12, 2k\right\}. 
\end{eqnarray}
Also, assume that $ k(n-k) \leq N$. Then, with probability at least $1 - \epsilon$, the sampled matrix will be finitely completable.
\end{theorem}

Observe that in the case of $1< k < n-1$, the assumption $ k(n-k) \leq N$ results that $n < N$ which is very important to check when we apply this theorem. In order to use Theorem \ref{Grassmannianprob}, we need the following lemma to obtain an upper bound on the rank of unfoldings of $\mathcal{U}$.

\begin{lemma}\label{rankunfsmCP}
Consider an arbitrary nonempty $I \subset \{1,\dots,d\}$ and recall that $r$ denotes the CP rank of $\mathcal{U}$. Then, $\text{rank} \left( \mathbf{\widetilde U}_{(I)} \right) \leq r$.
\end{lemma}

\begin{proof}
In order to show $\text{rank} \left( \mathbf{\widetilde U}_{(I)} \right) \leq r$, we show the existence of a CP decomposition of $\mathbf{\widetilde U}_{(I)}$ with $r$ components, i.e., we show that there exist $\mathbf{a}_{i}^{l} \in \mathbb{R}^{m_i}$ for $1 \leq i \leq 2$, $1 \leq l \leq r$, $m_1 \triangleq N_{I}$ and $m_2 \triangleq N_{\bar I}$ such that
\begin{eqnarray}\label{CPdecomunfold}
\mathbf{\widetilde U}_{(I)} = \sum_{l=1}^{r}  \mathbf{a}_{1}^{l} \otimes \mathbf{a}_{2}^{l}.
\end{eqnarray}

In order to do so, recall that since the CP rank of $\mathcal{U}$ is $r$, there exist $\mathbf{b}_{i}^{l} \in \mathbb{R}^{n_i}$ for $1 \leq i \leq d$ and $1 \leq l \leq r$ such that
\begin{eqnarray}\label{CPdecomcopy}
\mathcal{U} = \sum_{l=1}^{r}  \mathbf{b}_{1}^{l} \otimes \mathbf{b}_{2}^{l} \otimes \dots \otimes \mathbf{b}_{d}^{l}.
\end{eqnarray}
Define $\mathcal{A}_{1}^{l}=\mathbf{b}_{i_1}^{l} \otimes \dots \otimes \mathbf{b}_{i_{|I|}}^{l}$ and $\mathcal{A}_{2}^{l}=\mathbf{b}_{i_{|I|+1}}^{l} \otimes \dots \otimes \mathbf{b}_{i_d}^{l}$ for $1 \leq l \leq l$, where $I=\{{i_1},\dots,{i_{|I|}}\}$, $\bar I=\{{i_{|I|+1}},\dots,{i_d}\}$. Let $\mathbf{a}_{1}^{l}$ and $\mathbf{a}_{2}^{l}$ denote the vectorizations of $\mathcal{A}_{1}^{l}$ and $\mathcal{A}_{2}^{l} $ with the same bijective mappings  ${\widetilde M}_{I}: (x_{i_1},\dots,x_{i_{|I|}}) \rightarrow  \{1,2,\dots, N_I\}$ and ${\widetilde{{M}}}_{\bar I}: (x_{i_{|I|+1}},\ldots,x_{i_d}) \rightarrow  \{1,2,\dots, \bar N_{\bar I} \}$ of the unfolding $\mathbf{\widetilde U}_{(I)}$. Hence, there exist $\mathbf{a}_{i}^{l} \in \mathbb{R}^{m_i}$ for $1 \leq i \leq 2$, $1 \leq l \leq r$ such that \eqref{CPdecomunfold} holds.
\end{proof}

\begin{remark}\label{increasorder}
Assume that $k \leq k^{\prime} \leq \frac{n}{6}$, $l > \max\left\{12 \ \log \left( \frac{n}{\epsilon} \right) + 12, 2k^{\prime} \right\}$ and $ k^{\prime}(n-k^{\prime}) \leq N$. Then, we have $l > \max\left\{12 \ \log \left( \frac{n}{\epsilon} \right) + 12, 2k \right\}$ since $k \leq k^{\prime}$. Moreover, we have $ k(n-k) \leq N$ since $k+k^{\prime} < n$ which results $ k(n-k) < k^{\prime}(n-k^{\prime})$.
\end{remark}

\begin{lemma}\label{finunfoldappr}
Let $I=\{{i_1},\dots,{i_{|I|}}\}$ be an arbitrary nonempty and proper subset of $ \{1,\dots,d\}$. Assume that $|I| < \frac{d}{2}$ and $r \leq \frac{n}{6}$, where $r$ is the CP rank of the sampled tensor $\mathcal{U}$. Moreover, assume that each {\bf column} of $\mathbf{\widetilde U}_{(I)}$ is observed in at least $l$ entries, distributed uniformly at random and independently across entries, where
\begin{eqnarray}\label{unfoldingprobineq}
l > \max\left\{12 \ \log \left( \frac{N_I r}{\epsilon} \right) + 12, 2r \right\}. 
\end{eqnarray}
Then, with probability at least $1 - \epsilon$, the sampled tensor will be finitely completable.
\end{lemma}

\begin{proof}
According to Lemma \ref{rankunfsmCP}, $r_I \triangleq \text{rank} \left( \mathbf{\widetilde U}_{(I)} \right) \leq r$. Note that $N_I \leq n^{\frac{d-1}{2}} = \frac{n^{\frac{d+1}{2}}}{n} \leq \frac{N_{\bar I}}{n} $ which results that $r(N_I -r) \leq N_{\bar I}$. Furthermore, according to Remark \ref{increasorder}, we have $r_I(N_I -r_I) \leq N_{\bar I}$ and also
\begin{eqnarray}\label{unfoldingprobineqcopy}
l > \max\left\{12 \ \log \left( \frac{N_I r}{\epsilon} \right) + 12, 2r_I \right\}. 
\end{eqnarray}
Therefore, according to Theorem \ref{Grassmannianprob}, $\mathbf{\widetilde U}_{(I)}$ is finitely completable for an arbitrary value of $r_I $ that belongs to $\{1,\dots,r\}$ with probability at least $1 - \frac{\epsilon}{r}$. Hence, with probability at least $1 - \epsilon$, for all possible values of $r_I$, $\mathbf{\widetilde U}_{(I)}$ is finitely completable, i.e., $\mathbf{\widetilde U}_{(I)}$ is finitely completable. In order to complete to proof, it suffices to observe that finite completability of any of the unfoldings of $\mathcal{U}$ results the finite completability of $\mathcal{U}$.
\end{proof}

\begin{remark}
In the case of $|I| > \frac{d}{2}$ in Lemma \ref{finunfoldappr}, we can simply consider the transpose of $\mathbf{\widetilde U}_{(I)}$ to have the similar results.
\end{remark}

\begin{remark}\label{numbofsampmatapp}
Lemma \ref{finunfoldappr} requires
\begin{eqnarray}\label{unfoldtotsampcase}
n^{d-|I|} \max\left\{12 \ \log \left( \frac{n^{|I|} r}{\epsilon} \right) + 12, 2r \right\}
\end{eqnarray}
samples in total to ensure the finite completability of $\mathcal{U}$ with probability at least $1 - \epsilon$. Hence, the best bound on the total number of samples to ensure the finite completability with probability at least $1-\epsilon$ will be obtained when $|I|=\lfloor \frac{d-1}{2} \rfloor$, which is
\begin{eqnarray}\label{unfoldtotsampgen}
n^{\lceil \frac{d+1}{2} \rceil} \max\left\{12 \ \log \left( \frac{n^{\lfloor \frac{d-1}{2} \rfloor} r}{\epsilon} \right) + 12, 2r \right\}.
\end{eqnarray}
\end{remark}

\subsection{CP Approach}\label{subsCP}

In this section, we present an approach based on the tensor CP decomposition instead of unfolding. Conditions (i) and (ii) in Theorem \ref{mainThm} ensure finite completability with probability one. Here, using combinatorial methods, we derive a lower bound on the number of sampled entries, i.e., the sampling probability, which ensures conditions (i) and (ii) in Theorem \ref{mainThm} hold with high probability. We first provide a few lemmas from our previous works. Lemma \ref{genlemm} below is Lemma $5$ in \cite{ashraphijuo2}, which will be used later.

\begin{lemma}\label{genlemm}
Assume that $r^{\prime} \leq \frac{n}{6}$ and also each column of $\mathbf{\Omega}_{(1)}$ (first matricization of $\Omega$) includes at least $l$ nonzero entries, where 
\begin{eqnarray}\label{genminl1}
l > \max\left\{9 \ \log \left( \frac{n}{\epsilon} \right) + 3 \ \log \left( \frac{k}{\epsilon} \right) + 6, 2r^{\prime}\right\}. 
\end{eqnarray}
Let $\mathbf{\Omega}^{\prime}_{(1)}$ be an arbitrary set of $n -r^{\prime}$ columns of $\mathbf{\Omega}_{(1)}$. Then, with probability at least $1-\frac{\epsilon}{k}$, every subset $\mathbf{\Omega}^{\prime \prime}_{(1)}$ of columns of $\mathbf{\Omega}^{\prime}_{(1)}$ satisfies 
\begin{eqnarray}\label{genproper1}
m_{1}({\Omega}^{\prime \prime}) - r^{\prime} \geq t,
\end{eqnarray}
where $t$ is the number of columns of $\mathbf{\Omega}^{\prime \prime}_{(1)}$ and $m_{1}({\Omega}^{\prime \prime})$ is the number of nonzero rows of $\mathbf{\Omega}^{\prime \prime}_{(1)}$.
\end{lemma}

\begin{lemma}\label{combpin}
Let $j \in \{1,2,\dots,d-1\}$ be a fixed number and $I = \{1,2,\dots,j\}$. Consider an arbitrary set $\widetilde{\mathbf{\Omega}}^{\prime}_{(I)}$ of $n -r^{\prime}$ columns of $\widetilde{\mathbf{\Omega}}_{(I)}$, where $r^{\prime} \leq r \leq \frac{n}{6}$. Assume that $n > 200$, and also each column of $\widetilde{\mathbf{\Omega}}_{(I)}$ includes at least $l$ nonzero entries, where 
\begin{eqnarray}\label{minlpn}
l > \max\left\{27 \ \log \left( \frac{n}{\epsilon} \right) + 9 \ \log \left( \frac{2k}{\epsilon} \right) + 18, 6r^{\prime}\right\},
\end{eqnarray}
where $k \leq r$. Then, with probability at least $1-\frac{\epsilon}{2k}$, each column of $\widetilde{\mathbf{\Omega}}^{\prime}_{(I)}$ includes more than $l_0 \triangleq \max\left\{9 \ \log \left( \frac{n}{\epsilon} \right) + 3 \ \log \left( \frac{2k}{\epsilon} \right) + 6, 2r^{\prime}\right\}$ observed entries of $\Omega$ with different values of the $i$-th coordinate, i.e., the $i$-th matricization of the tensor $\Omega^{\prime}$ that corresponds to $\widetilde{\mathbf{\Omega}}^{\prime}_{(I)}$ includes more than $l_0$ nonzero rows, $1 \leq i \leq j$.
\end{lemma}

\begin{proof}
The proof is omitted due to the similarity to the proof for Lemma $9$ in \cite{ashraphijuo3}.
\end{proof}

The following lemma is Lemma $8$ in \cite{ashraphijuo}, which states that if the property in Lemma \ref{genlemm} holds for the sampling pattern $\Omega$, it will be satisfied for $\breve{\Omega}$ as well.

\begin{lemma}\label{Omega}
Let $r^{\prime}$ be a given nonnegative integer and $1 \leq i \leq j \leq d-1$ and $I = \{1,2,\dots,j\}$. Assume that there exists an $n^j \times (n -r^{\prime})$ matrix $\widetilde{\mathbf{\Omega}}^{\prime}_{(I)}$ composed  of $n-r^{\prime}$ columns of $\widetilde{\mathbf{\Omega}}_{(I)}$ such that each column of $\widetilde{\mathbf{\Omega}}^{\prime}_{(I)}$ includes at least $r^{\prime}+1$ nonzero entries and satisfies the following property:
\begin{itemize}
\item Denote an $n^j \times t$  matrix (for any $1 \leq t \leq n-r^{\prime}$) composed of any $t$ columns of $\widetilde{\mathbf{\Omega}}^{\prime}_{(I)}$ by $\widetilde{\mathbf{\Omega}}^{\prime \prime}_{(I)}$. Then 
\begin{eqnarray}\label{proper233}
m_{i}({\Omega}^{\prime \prime}) -r^{\prime}  \geq t,
\end{eqnarray}
where $\widetilde{\mathbf{\Omega}}^{\prime \prime}_{(I)}$ is the unfolding of ${\Omega}^{\prime \prime}$ corresponding to the set $I$.
\end{itemize}
Then, there exists an $n^j \times (n -r^{\prime})$ matrix $\widetilde{\mathbf{\breve{\Omega}}}^{\prime}_{(I)}$ such that: each column has exactly $r^{\prime}+1$ entries equal to one, and if $\widetilde{\mathbf{\breve{\Omega}}}^{\prime}_{(I)}(x,y)=1$ then we have $\widetilde{\mathbf{\Omega}}^{\prime}_{(I)}(x,y)=1$. Moreover, $\widetilde{\mathbf{\breve{\Omega}}}^{\prime}_{(I)}$ satisfies the above-mentioned property.
\end{lemma}

\begin{lemma}\label{lemman3}
Assume that $ n > 200$, $1 \leq i \leq j \leq d-1$ and $I = \{1,2,\dots,j\}$. Consider $r$ disjoint sets $\widetilde{\mathbf{\Omega}}^{{\prime}}_{l_{(I)}}$, each with $n -r^{\prime}_i$ columns of $\widetilde{\mathbf{\Omega}}_{(I)}$ for $1 \leq l \leq r$, where $r^{\prime}_i \leq r \leq \frac{n}{6}$. Let $\widetilde{\mathbf{\Omega}}^{{\prime}}_{{(I)}}$ denote the union of all $r$ sets of columns $\widetilde{\mathbf{\Omega}}^{{\prime}}_{l_{(I)}}$'s. Assume that each column of $\widetilde{\mathbf{\Omega}}_{(I)}$ includes at least $l$ nonzero entries, where 
\begin{eqnarray}\label{minlforset3}
l > \max\left\{27 \ \log \left( \frac{n}{\epsilon} \right) + 9 \ \log \left( \frac{2rk}{\epsilon} \right) + 18, 6r_i^{\prime}\right\}. 
\end{eqnarray}
Then, there exists an $n^j \times r(n -r^{\prime}_i)$ matrix $\widetilde{\breve{\mathbf{\Omega}}}^{\prime}_{(I)}$ such that: each column has exactly $r^{\prime}_i+1$ entries equal to one, and if $\widetilde{\breve{\mathbf{\Omega}}}^{\prime}_{(I)}(x,y)=1$ then we have $\widetilde{\mathbf{\Omega}}^{\prime}_{(I)}(x,y)=1$ and also it satisfies the following property: with probability at least $1-\frac{\epsilon }{k}$, every subset $\widetilde{\breve{\mathbf{\Omega}}}^{\prime \prime}_{(I)}$ of columns of $\widetilde{\breve{\mathbf{\Omega}}}^{\prime}_{(I)}$ satisfies the following inequality
\begin{eqnarray}\label{proper3}
r \left(m_{i}(\breve{{\Omega}}^{\prime \prime}) -r^{\prime}_i \right) \geq t,
\end{eqnarray}
where $t$ is the number of columns of $\widetilde{\breve{\mathbf{\Omega}}}^{\prime \prime}_{(I)}$ and $\breve{\Omega}^{\prime \prime}$ is the tensor corresponding to unfolding $\widetilde{\breve{\mathbf{\Omega}}}^{\prime \prime}_{(I)}$.
\end{lemma}

\begin{proof}
We first claim that with probability at least $1-\frac{\epsilon }{kr}$, every subset $\widetilde{\mathbf{\Omega}}^{{\prime \prime}}_{l_{(I)}}$ of columns of $\widetilde{\mathbf{\Omega}}^{{\prime}}_{l_{(I)}}$ satisfies
\begin{eqnarray}\label{proper3jj}
m_{i}({{\Omega}}^{\prime \prime}_l) -r^{\prime}_i  \geq t,
\end{eqnarray}
where $t$ is the number of columns of $\widetilde{\mathbf{\Omega}}^{{\prime \prime}}_{l_{(I)}}$ and ${\Omega}^{\prime \prime}_l$ is the tensor corresponding to unfolding $\widetilde{\mathbf{\Omega}}^{{\prime \prime}}_{l_{(I)}}$. For simplicity we denote the above-mentioned property by Property I. According to Lemma \ref{combpin}, with probability at least $1-\frac{\epsilon }{2kr}$, the $i$-th matricization of the tensor ${\Omega}^{\prime \prime}_l$ includes more than $ \max\left\{9 \ \log \left( \frac{n}{\epsilon} \right) + 3 \ \log \left( \frac{2kr}{\epsilon} \right) + 6, 2r^{\prime}_i \right\}$ nonzero rows, $1 \leq i \leq j$, and we denote this property by Property II. On the other hand, given that Property II holds for ${\Omega}^{\prime \prime}_l$ and according to Lemma \ref{genlemm}, with probability at least $1-\frac{\epsilon }{2kr}$, Property I holds for ${\Omega}^{\prime \prime}_l$ as well. Hence, with probability at least $1-\frac{\epsilon }{kr}$, Property I holds for ${\Omega}^{\prime \prime}_l$, which completes the proof our earlier claim.

Consequently, according to Lemma \ref{Omega}, with probability at least $1-\frac{\epsilon }{kr}$, there exists an $n^j \times (n -r^{\prime}_i)$ matrix $\widetilde{\breve{\mathbf{\Omega}}}^{\prime}_{l_{(I)}}$ such that: each column has exactly $r^{\prime}_i+1$ entries equal to one, and if $\widetilde{\breve{\mathbf{\Omega}}}^{\prime}_{l_{(I)}}(x,y)=1$ then we have $\widetilde{\mathbf{\Omega}}^{\prime}_{l_{(I)}}(x,y)=1$ and also $\breve{\Omega}^{\prime}_l$ satisfies Property I. Finally define $\widetilde{\breve{\mathbf{\Omega}}}^{\prime}_{{(I)}} \triangleq \left[\widetilde{\breve{\mathbf{\Omega}}}^{\prime}_{1_{(I)}}|\widetilde{\breve{\mathbf{\Omega}}}^{\prime}_{2_{(I)}}| \dots | \widetilde{\breve{\mathbf{\Omega}}}^{\prime}_{r_{(I)}}\right]$. Since each $\breve{\Omega}^{\prime}_l$ satisfies Property I with probability at least $1-\frac{\epsilon }{kr}$, all $\breve{\Omega}^{\prime}_l$'s satisfy Property I with probability at least $1-\frac{\epsilon }{k}$, simultaneously. Consider an arbitrary subset $\widetilde{\breve{\mathbf{\Omega}}}^{\prime \prime}_{{(I)}}$ of columns of $\widetilde{\breve{\mathbf{\Omega}}}^{\prime }_{{(I)}}$. Let $\widetilde{\breve{\mathbf{\Omega}}}^{\prime \prime}_{l_{(I)}}$ denote those columns of $\widetilde{\breve{\mathbf{\Omega}}}^{\prime \prime}_{{(I)}}$ that belong to $\widetilde{\breve{\mathbf{\Omega}}}^{\prime}_{l_{(I)}}$ and define $t_l$ as the number of columns of $\widetilde{\breve{\mathbf{\Omega}}}^{\prime \prime}_{l_{(I)}}$, $1 \leq l \leq r$, and define $t$ as the number of columns of $\widetilde{\breve{\mathbf{\Omega}}}^{\prime \prime}_{{(I)}}$. Without loss of generality, assume that $t_1 \leq t_2 \leq \dots \leq t_{r}$. Also, assume that all $\breve{\Omega}^{\prime}_l$'s satisfy Property I. Hence, we have
\begin{eqnarray}\label{proper7} 
t = \sum_{l=1}^{r} t_l \leq r t_r \leq r \left( m_{i}({\breve{\Omega}}^{\prime \prime}_{r}) -r_i^{\prime} \right) \leq r \left(m_{i}(\breve{\Omega}^{\prime \prime}) -r_i^{\prime} \right).
\end{eqnarray}
\end{proof}

\begin{theorem}\label{mainthmprobfin}
Assume that $d>2$, $ n > \max \{ 200, r(d-2)\}$, $r \leq \frac{n}{6}$ and $I = \{1,2,\dots,d-2\}$.  Assume that each column of $\widetilde{\mathbf{\Omega}}_{(I)}$ includes at least $l$ nonzero entries, where 
\begin{eqnarray}\label{minl3j}
l > \max\left\{27 \ \log \left( \frac{n}{\epsilon} \right) + 9 \ \log \left( \frac{2r(d-2)}{\epsilon} \right) + 18, 6r \right\}. 
\end{eqnarray}
Then, with probability at least $1-\epsilon$, for almost every $\mathcal{U} \in \mathbb{R}^{\overbrace {n \times \dots \times n}^{d}}$, there exist only finitely many completions of the sampled tensor $\mathcal{U}$ with CP rank $r$.
\end{theorem}

\begin{proof}
Define the $(d-1)$-way tensor $\mathcal{U}^{\prime} \in \mathbb{R}^{\footnotesize {\overbrace {n \times \dots \times n}^{d-2}} \times {n}^2}$ which is obtained through merging the $(d-1)$-th and $d$-th dimensions of the tensor $\mathcal{U}$. Observe that the finiteness of the number of completions of the tensor $\mathcal{U}^{\prime}$ of rank $r$ ensures the finiteness of the number of completions of the tensor $\mathcal{U}$ of rank $r$. For notational simplicity, let $\Omega$ and $\breve{\Omega}$ denote the $(d-1)$-way sampling pattern and constraint tensors corresponding to $\mathcal{U}^{\prime}$, respectively. In order to complete the proof it suffices to show with probability at least $1 - \epsilon$, conditions (i) and (ii) in Theorem \ref{mainThm} hold for this modified $(d-1)$-way tensor.

Now, we apply Lemma \ref{lemman3} for each of the numbers $r_1^{\prime}=1, \dots , r_{d-3}^{\prime}=1, r_{d-2}^{\prime}=r$. Also, note that since $n > r(d-2)$ we conclude $n^2 > r(n-r) + (d-3)r(n-1)$, and therefore $\widetilde{{\mathbf{\Omega}}}_{{(I)}}$ includes more than  $r(n-r) + (d-3)r(n-1)$ columns. According to Lemma \ref{lemman3}, there exist $\breve{\Omega}^{\prime}_i$ for $1 \leq i \leq d-2$ such that: (i) each column of $\widetilde{\breve{\mathbf{\Omega}}}^{\prime}_{i_{(I)}}$ includes $r_i^{\prime}+1$ nonzero entries for $1 \leq i \leq d-2$, and if $\widetilde{\breve{\mathbf{\Omega}}}^{\prime}_{i_{(I)}}(x,y)=1$ then we have $\widetilde{\mathbf{\Omega}}^{\prime}_{i_{(I)}}(x,y)=1$, (ii) $\widetilde{\breve{\mathbf{\Omega}}}^{\prime}_{i_{(I)}}$ includes $r(n-1)$ and $r(n-r)$ columns for $1 \leq i \leq d-3$ and $i=d-2$, respectively, (iii) with probability at least $1-\frac{\epsilon }{d-2}$, every subset $\widetilde{\breve{\mathbf{\Omega}}}^{\prime \prime}_{i_{(I)}}$ of columns of $\widetilde{\breve{\mathbf{\Omega}}}^{\prime}_{i_{(I)}}$ satisfies \eqref{proper3} for $r_1^{\prime}=1, \dots , r_{d-3}^{\prime}=1, r_{d-2}^{\prime}=r$.

Recall that each column of $\widetilde{{\mathbf{\Omega}}}_{{(I)}}$ includes $r+1$ nonzero entries, and therefore for $1 \leq i \leq d-3$ that we have $r_i^{\prime}+1=2$, the column of $\widetilde{{\mathbf{\Omega}}}_{{(I)}}$ corresponding to an column of $\widetilde{\breve{\mathbf{\Omega}}}^{\prime}_{i_{(I)}}$ has $r-1$ more nonzero entries.

Observe that $ \max\left\{9 \ \log \left( \frac{n}{\epsilon} \right) + 3 \ \log \left( \frac{2kr}{\epsilon} \right) + 6, 2r \right\} \geq 2r \geq (r-1)+2$. According to Lemma \ref{combpin} and given \eqref{minl3j}, for each column of $\widetilde{\breve{\mathbf{\Omega}}}^{\prime}_{i_{(I)}}$ there exists another $r-1$ zero entries $(x_s,y_s)$ for $s \in \{1,\dots,r-1\}$ in different rows of the $(d-2)$-th matricization of $\breve{\Omega}^{\prime}_i$ from the two nonzero entries such that $\widetilde{\mathbf{\Omega}}^{\prime}_{i_{(I)}}(x_s,y_s)=1$, $i=1,\dots,d-3$. Hence, for each column of $\widetilde{\breve{\mathbf{\Omega}}}^{\prime}_{i_{(I)}}$, we substitute it with the column of $\widetilde{{\mathbf{\Omega}}}_{{(I)}}$  that the value of such $r-1$ entries (in different rows of the $(d-2)$-th matricization of $\breve{\Omega}^{\prime}_i$) is $1$ instead of $0$ , $i=1,\dots,d-3$. Therefore, each column of $\widetilde{\breve{\mathbf{\Omega}}}^{\prime}_{i_{(I)}}$ includes exactly $r+1$ nonzero entries for $1 \leq i \leq d-3$ such that if $\widetilde{\breve{\mathbf{\Omega}}}^{\prime}_{i_{(I)}}(x,y)=1$ then we have $\widetilde{\mathbf{\Omega}}^{\prime}_{i_{(I)}}(x,y)=1$.

Let $\widetilde{\breve{\mathbf{\Omega}}}^{\prime}_{{(I)}}= \left[ \widetilde{\breve{\mathbf{\Omega}}}^{\prime}_{1_{(I)}} | \dots | \widetilde{\breve{\mathbf{\Omega}}}^{\prime}_{{d-2}_{(I)}} \right]$, which includes $r(n-r) + (d-3)r(n-1)$ columns. Therefore, ${\breve{\Omega}}^{\prime} \in  \mathbb{R}^{\footnotesize {\overbrace {n \times \dots \times n}^{d-2}} \times t}$ is a subtensor of the constraint tensor such that $t = r(\sum_{i=1}^{d-2} n) - r^2 - r(d-3)$ and also and with probability at least $1-\epsilon$, every subset $\widetilde{\breve{\mathbf{\Omega}}}^{\prime \prime}_{i_{(I)}}$ of columns of $\widetilde{\breve{\mathbf{\Omega}}}^{\prime}_{i_{(I)}}$ satisfies \eqref{proper3} for $r_1^{\prime}=1, \dots , r_{d-3}^{\prime}=1, r_{d-2}^{\prime}=r$, simultaneously for $i = 1,\dots,d-2$.

Consider an arbitrary subset $\widetilde{\breve{\mathbf{\Omega}}}^{\prime \prime}_{{(I)}}$ of columns of $\widetilde{\breve{\mathbf{\Omega}}}^{\prime }_{{(I)}}$. Let $\widetilde{\breve{\mathbf{\Omega}}}^{\prime \prime}_{i_{(I)}}$ denote those columns of $\widetilde{\breve{\mathbf{\Omega}}}^{\prime \prime}_{{(I)}}$ that belong to $\widetilde{\breve{\mathbf{\Omega}}}^{\prime}_{i_{(I)}}$ and define $t_i$ as the number of columns of $\widetilde{\breve{\mathbf{\Omega}}}^{\prime \prime}_{i_{(I)}}$, $1 \leq i \leq d-2$, and define $t$ as the number of columns of $\widetilde{\breve{\mathbf{\Omega}}}^{\prime \prime}_{{(I)}}$. Also, assume that all $\breve{\Omega}^{\prime}_l$'s satisfy \eqref{proper3} for $r_1^{\prime}=1, \dots , r_{d-3}^{\prime}=1, r_{d-2}^{\prime}=r$. Then, we have the following two scenarios:

(i) $t_{d-2} =0$: Hence, we have $t = \sum_{i=1}^{d-3} t_i$. Moreover, we have
\begin{eqnarray}\label{proper09} 
t_i \leq r \left( m_{i}({\breve{\Omega}}^{\prime \prime}_i) -1 \right)^+ \leq r \left( m_{i}({\breve{\Omega}}^{\prime \prime}) -1 \right)^+ =  r \left( m_{i}({\breve{\Omega}}^{\prime \prime}) -1 \right),
\end{eqnarray}
for $1 \leq i \leq  d-3$. Recall that each column of $\widetilde{\breve{\mathbf{\Omega}}}^{\prime}_{i_{(I)}}$ includes at least $r$ nonzero entries in different rows of the $(d-2)$-th matricization of $\breve{\Omega}^{\prime}_i$ for $1 \leq i \leq  d-3$. On the other hand, since $\widetilde{\breve{\mathbf{\Omega}}}^{\prime \prime}_{{(I)}}$ includes at least one column of $\left[ \widetilde{\breve{\mathbf{\Omega}}}^{\prime}_{1_{(I)}} | \dots | \widetilde{\breve{\mathbf{\Omega}}}^{\prime}_{{d-3}_{(I)}} \right]$ (recall that $t_{d-2}=0$), we have
\begin{eqnarray}\label{proper10} 
r \leq  m_{d-2}({\breve{\Omega}}^{\prime \prime}),
\end{eqnarray}
which also results that $\text{min} \left\{ \text{max} \left\{ m_1 (\breve{\Omega}^{\prime \prime}) , \dots ,m_{d-2}(\breve{\Omega}^{\prime \prime}) \right\} , r \right\} = r$.

Therefore, having \eqref{proper09} and \eqref{proper10}, we conclude
\begin{eqnarray}\label{proper11} 
 t = \sum_{i=1}^{d-3} t_i \leq  \sum_{i=1}^{d-3} r \left( m_{i}({\breve{\Omega}}^{\prime \prime}) - 1 \right) \leq \sum_{i=1}^{d-3} r \left( m_{i}({\breve{\Omega}}^{\prime \prime}) - 1 \right) + r \left( m_{d-2}({\breve{\Omega}}^{\prime \prime}) - r  \right)  \\ \nonumber  = r \left( \sum_{i=1}^{d-2} m_{i}({\breve{\Omega}}^{\prime \prime}) \right) - r^2 - r(d-3)  \\ \nonumber = r \left( \left( \sum_{i=1}^{d-2} m_i (\breve{\Omega}^{\prime \prime}) \right) -   \text{min} \left\{ \text{max} \left\{ m_1 (\breve{\Omega}^{\prime \prime}) , \dots ,m_{d-2}(\breve{\Omega}^{\prime \prime}) \right\} ,r \right\} - (d-3) \right).
\end{eqnarray}

(ii) $t_{d-2} > 0$: Hence, we have
\begin{eqnarray}\label{proper12} 
t_{d-2} \leq r \left( m_{d-2}({\breve{\Omega}}^{\prime \prime}_{d-2}) -r \right)  \leq  r \left( m_{d-2}({\breve{\Omega}}^{\prime \prime}) -r \right),
\end{eqnarray}
which also results that $\text{min} \left\{ \text{max} \left\{ m_1 (\breve{\Omega}^{\prime \prime}) , \dots ,m_{d-2}(\breve{\Omega}^{\prime \prime}) \right\} , r \right\} = r$. Moreover, similar to scenario (i), \eqref{proper09} holds. Therefore, having \eqref{proper09} and \eqref{proper12}, we conclude
\begin{eqnarray}\label{proper13} 
 t = \sum_{i=1}^{d-2} t_i \leq  \sum_{i=1}^{d-3} r \left( m_{i}({\breve{\Omega}}^{\prime \prime}) - 1 \right) + r \left( m_{d-2}({\breve{\Omega}}^{\prime \prime}) -r \right)    = r \left( \sum_{i=1}^{d-2} m_{i}({\breve{\Omega}}^{\prime \prime}) \right) - r^2 - r(d-3)  \\ \nonumber = r \left( \left( \sum_{i=1}^{d-2} m_i (\breve{\Omega}^{\prime \prime}) \right) -   \text{min} \left\{ \text{max} \left\{ m_1 (\breve{\Omega}^{\prime \prime}) , \dots ,m_{d-2}(\breve{\Omega}^{\prime \prime}) \right\} ,r \right\} - (d-3) \right).
\end{eqnarray}
\end{proof}

Note that for the general values of $n_1$, $\dots$, $n_d$ the same proof will still work, but instead of the assumption $ n > \max \{ 200, r(d-2)\}$, we need another assumption in terms of $n_1$, $\dots$, $n_d$  to ensure that the corresponding unfolding has enough number of columns.

\begin{remark}\label{remsam}
A tensor $\mathcal{U}$ that satisfies the properties in the statement of Theorem \ref{mainthmprobfin} requires 
\begin{eqnarray}\label{rremsam}
n^2 \max\left\{27 \ \log \left( \frac{n}{\epsilon} \right) + 9 \ \log \left( \frac{2r(d-2)}{\epsilon} \right) + 18, 6r \right\}
\end{eqnarray}
samples to be finitely completable with probability at least $1-\epsilon$, which is lower than the number of samples required by the unfolding approach given in \eqref{unfoldtotsampgen} by orders of magnitude.
\end{remark}

The following lemma is taken from \cite{ashraphijuo} and is used in Lemma \ref{prsCPfin} to derive a lower bound on the sampling probability that results \eqref{minl3j} with high probability.

\begin{lemma}\label{azumares}
Consider a vector with $n$ entries where each entry is observed with  probability  $p$  independently from the other entries. If $p > p^{\prime} = \frac{k}{n} + \frac{1}{\sqrt[4]{n}}$, then with probability  at least $\left(1-\exp(-\frac{\sqrt{n}}{2})\right)$, more than $k$ entries are observed.
\end{lemma}

\begin{lemma}\label{prsCPfin}
Assume that $d>2$, $ n > \max \{ 200, r(d-2)\}$ and $r \leq \frac{n}{6}$. Moreover, assume that the sampling probability satisfies
\begin{eqnarray}\label{prsCPbdfin}
p > \frac{1}{n^{d-2}} \max\left\{27 \ \log \left( \frac{n}{\epsilon} \right) + 9 \ \log \left( \frac{2r(d-2)}{\epsilon} \right) + 18, 6r \right\} + \frac{1}{\sqrt[4]{n^{d-2}}}.
\end{eqnarray}
Then, with probability at least $ (1- \epsilon) \left( 1-\exp(-\frac{\sqrt{n^{d-2}}}{2}) \right)^{n^2}$, $\mathcal{U}$ is finitely completable.
\end{lemma}

\begin{proof}
According to Lemma \ref{azumares}, \eqref{prsCPbdfin} results that each column of $\widetilde{\mathbf{\Omega}}_{(I)}$ includes at least $l$ nonzero entries, where  $I = \{1,2,\dots,d-2\}$ and $l$ satisfies \eqref{minl3j} with probability at least $\left( 1-\exp(-\frac{\sqrt{n^{d-2}}}{2}) \right)$. Therefore, with probability at least $\left( 1-\exp(-\frac{\sqrt{n^{d-2}}}{2}) \right)^{n^2}$, all $n^2$ columns of $\widetilde{\mathbf{\Omega}}_{(I)}$ satisfy \eqref{minl3j}. Hence, according to Theorem \ref{mainthmprobfin}, with probability at least $ (1- \epsilon) \left( 1-\exp(-\frac{\sqrt{n^{d-2}}}{2}) \right)^{n^2}$, $\mathcal{U}$ is finitely completable.
\end{proof}

\section{Deterministic and Probabilistic Conditions for Unique Completability}\label{secuni}

In this section, we are interested in characterizing the deterministic and probabilistic conditions on the sampling pattern for unique completability. In previous sections we characterized the corresponding conditions for finite completability in Theorem \ref{mainThm} and Theorem \ref{mainthmprobfin}. However, for matrix and tensor completion problems, finite completability does not necessarily imply unique completability  \cite{ashraphijuo}. In this section, we add some additional mild restrictions on $\Omega$ in the statement of Theorem \ref{mainThm} to ensure unique completability (deterministic) and also increase the number of samples  given in the statement of Theorem \ref{mainthmprobfin} mildly to ensure unique completability with high probability (probabilistic). As the first step of this procedure, we use the following lemma for minimally algebraically dependent polynomials to obtain the variables involved in these polynomials uniquely. Hence, by obtaining all entries of the CP decomposition of the sampled tensor $\mathcal{U}$ we can show the uniqueness of $\mathcal{U}$.

The following lemma is a re-statement of Lemma $9$ in \cite{ashraphijuo}.

\begin{lemma}\label{uniindpoldet}
Suppose that Assumption $1$ holds. Let ${\breve{\Omega}}^{\prime} \in \mathbb{R}^{n_1 \times n_2 \times \dots \times n_{d-1} \times t}$ be an arbitrary subtensor of the constraint tensor ${\breve{\Omega}}$. Assume that polynomials in $\mathcal{P}({\breve{\Omega}}^{\prime})$ are minimally algebraically dependent. Then, all variables (unknown entries) of $\mathbf{A}_1, \dots , \mathbf{A}_{d-1}$ that are involved in $\mathcal{P}({\breve{\Omega}}^{\prime})$ can be determined uniquely.
\end{lemma}

Condition (i) in Theorem \ref{mainthmdetuni} results in $  r(\sum_{i=1}^{d-1} n_i) - r^2 - r(d-2) $ algebraically independent polynomials in terms of the entries of $\mathbf{A}_1, \dots , \mathbf{A}_{d-1}$, i.e., results in finite completability. Hence, adding a single polynomial corresponding to any observed entry to these $  r(\sum_{i=1}^{d-1} n_i) - r^2 - r(d-2) $ algebraically independent polynomials results in a set of algebraically dependent polynomials. Then, according to Lemma \ref{uniindpoldet} a subset of the entries of $\mathbf{A}_1, \dots , \mathbf{A}_{d-1}$ can be determined uniquely and these additional polynomials are captured in the structure of condition (ii) such that all entries of CP decomposition can be determined uniquely.

\begin{theorem}\label{mainthmdetuni}
Suppose that Assumption $1$ holds. Also, assume that there exist disjoint subtensors ${\breve{\Omega}}^{\prime} \in \mathbb{R}^{n_1 \times n_2 \times \cdots \times n_{d-1} \times t}$ and ${\breve{\Omega}}^{{\prime}^i} \in \mathbb{R}^{n_1 \times n_2 \times \cdots \times n_{d-1} \times t_i}$ (for $1 \leq i \leq 2d-2$) of the constraint tensor such that the following conditions hold:

(i) $t = r(\sum_{i=1}^{d-1} n_i) - r^2 - r(d-2)$ and for any $ t^{\prime} \in \{1,\dots,t\}$ and any subtensor ${\breve{\Omega}}^{\prime \prime} \in \mathbb{R}^{n_1 \times n_2 \times \cdots \times n_{d-1} \times t^{\prime}}$ of ${\breve{\Omega}}^{\prime}$,  \eqref{ineqindppoly} holds.

(ii) for $i \in \{1,\dots,d-1\}$ we have $t_i = n_i-1 $ and for $i \in \{d,\dots,2d-2\}$ we have $t_i = n_{i-d+1} - r$. Also, for any $t_i^{\prime} \in \{1,\dots,t_i\}$ and any subtensor ${\breve{\Omega}}^{{\prime \prime}^i} \in \mathbb{R}^{n_1 \times n_2 \times \cdots \times n_{d-1} \times t_i^{\prime}}$ of the tensor ${\breve{\Omega}}^{{\prime}^i}$, the following inequalities hold
\begin{eqnarray}\label{ineqpuclm}
m_i({\breve{\Omega}}^{{\prime \prime}^i}) - 1 \geq t_i^{\prime} , \ \ \ \ \ \ \ \ \ \ \ \ \ \ \ \ \ \ \ \  \ \text{for} \ i \in \{1,\dots,d-1\},
\end{eqnarray}
and 
\begin{eqnarray}\label{ineqpuclmpr}
m_{i-d+1}({\breve{\Omega}}^{{\prime \prime}^i}) - r \geq t_i^{\prime},  \ \ \ \ \  \ \ \ \ \  \ \ \ \ \  \ \text{for} \ i \in \{d,\dots,2d-2\}.
\end{eqnarray}
Then, for almost every $\mathcal{U}$, there exists only a unique tensor that fits in the sampled tensor $\mathcal{U}$, and has CP rank $r$.
\end{theorem}

\begin{proof}
As we showed in the proof of Theorem \ref{mainThm}, $\mathcal{P}(\breve{\Omega}^{\prime})$ includes $t = r(\sum_{i=1}^{d-1} n_i) - r^2 - r(d-2)$ algebraically independent polynomials which results the finite completability of the sampled tensor $\mathcal{U}$. Let $\{ p_1, \dots , p_t \}$ denote these $t$ algebraically independent polynomials in $\mathcal{P}(\breve{\Omega}^{\prime})$. Now, having $\{ p_1, \dots , p_t \}$ and $\mathcal{P}(\breve{\Omega}^{{\prime}^i})$ for  $1 \leq i \leq 2d-2$, and using Lemma \ref{uniindpoldet} several times, we show the unique completability. Recall that $t$ is the number of total variables among the polynomials, and therefore union of any polynomial $p_0$ and $\{ p_1, \dots , p_t \}$ is a set of algebraically dependent polynomials. Hence, there exists a set of polynomials $\mathcal{P}(\breve{\Omega}^{\prime \prime})$ such that $\mathcal{P}(\breve{\Omega}^{\prime \prime}) \subset \{ p_1, \dots , p_t \}$ and also polynomials in $\mathcal{P}(\breve{\Omega}^{\prime \prime}) \cup p_0$ are minimally algebraically dependent polynomials. Therefore, according to Lemma \ref{uniindpoldet}, all variables involved in the polynomials $\mathcal{P}(\breve{\Omega}^{\prime \prime}) \cup p_0$ can be determined uniquely, and consequently, all variables involved in $p_0$ can be determined uniquely.

We can repeat the above procedure for any polynomial $p_0 \in \mathcal{P} ({\breve{\Omega}}^{{\prime}^i})$ to determine the involved variables uniquely with the help of $\{ p_1, \dots , p_t \}$, $i = 1,\dots,2d-2$. Hence, for any polynomial $p_0 \in \mathcal{P} ({\breve{\Omega}}^{{\prime}^i})$ or $p_0 \in \mathcal{P} ({\breve{\Omega}}^{{\prime}^{i+d-1}})$, we obtain $r$ degree-$1$ polynomials in terms of the entries of $\mathbf{A}_i$ but some of the entries of CP decomposition are elements of the $\mathbf{Q}_i$  matrices (in the statement of Lemma \ref{patternCP}), $i = 1,\dots,d-1$. In order to complete the proof, we need to show that condition (ii) with the above procedure using $\{ p_1, \dots , p_t \}$ results in obtaining all variables uniquely. In particular, we show that repeating the described procedure for any of the polynomials in $\mathcal{P} ({\breve{\Omega}}^{{\prime}^i})$ and $\mathcal{P} ({\breve{\Omega}}^{{\prime}^{i+d-1}})$ result in obtaining all variables of the $i$-th element of CP decomposition uniquely.

According to Lemma \ref{patternCP}, we have the following two scenarios for any $i \in  \{1,\dots,d-1\}$:

(i) $\mathbf{Q}_i \in \mathbb{R}^{r \times r}$: In this case, condition (ii) for ${\breve{\Omega}}^{{\prime \prime}^{i+d-1}}$ and for any $t_{i+d-1}^{\prime} \in \{1,\dots,n_i - r\}$ and any subtensor ${\breve{\Omega}}^{{\prime \prime}^{i+d-1}} \in \mathbb{R}^{n_1 \times n_2 \times \cdots \times n_{d-1} \times t_{i+d-1}^{\prime}}$ of the tensor ${\breve{\Omega}}^{{\prime}^{i+d-1}}$ results 
\begin{eqnarray}
r m_i({\breve{\Omega}}^{{\prime \prime}^{i+d-1}}) - r^2 \geq r t_{i+d-1}^{\prime}.
\end{eqnarray}
Note that $r t_{i+d-1}^{\prime}$ is the number of polynomials from the above mentioned procedure corresponding  to ${\breve{\Omega}}^{{\prime \prime}^{i+d-1}}$ and $r m_i({\breve{\Omega}}^{{\prime \prime}^{i+d-1}})$ denotes the number of involved entries of $\mathbf{A}_i$ in these polynomials, and therefore $r m_i({\breve{\Omega}}^{{\prime \prime}^{i+d-1}}) - r^2$ is the number of involved variables of $\mathbf{A}_i$ in these polynomials. As a result, according to Fact $2$, given $\mathcal{P} ({\breve{\Omega}}^{{\prime}^{i+d-1}})$ and $\{ p_1, \dots , p_t \}$, the mentioned procedure results in $rn_i - r^2$ algebraically independent degree-$1$ polynomials in terms of the unknown entries of $\mathbf{A}_i$. Therefore, $\mathbf{A}_i$ can be determined uniquely.

(ii) $\mathbf{Q}_i \in \mathbb{R}^{1 \times r}$: Similar to scenario (i), condition (ii) for ${\breve{\Omega}}^{{\prime \prime}^{i}}$ and for any $t_{i}^{\prime} \in \{1,\dots,n_i \}$ and any subtensor ${\breve{\Omega}}^{{\prime \prime}^{i}} \in \mathbb{R}^{n_1 \times n_2 \times \cdots \times n_{d-1} \times t_{i}^{\prime}}$ of the tensor ${\breve{\Omega}}^{{\prime}^{i}}$ results
\begin{eqnarray}
r m_i({\breve{\Omega}}^{{\prime \prime}^i}) - r \geq r t_i^{\prime},
\end{eqnarray}
and therefore similar to the previous scenario, $\mathbf{A}_i$ can be determined uniquely.
\end{proof}

In Theorem \ref{mainthmdetuni}, we obtained the deterministic condition on the sampling pattern for unique completability. Note that Condition (i) in Theorem \ref{mainthmdetuni} is the same condition for finite completability. 

In the remainder of this section, we are interested in characterizing the probabilistic conditions on the number of samples to ensure unique completability with high probability. For the sake of simplicity, as in Section \ref{sec3} we consider the sampled tensor $\mathcal{U} \in \mathbb{R}^{\overbrace {n \times \dots \times n}^{d}}$. Recall that for the general values of $n_1$, $\dots$, $n_d$ the same proof will still work, but instead of assumption $ n > \max \{ 200, (r+2)(d-2)\}$, we need another assumption in terms of $n_1$, $\dots$, $n_d$.

\begin{theorem}\label{mainthmprobuni}
Assume that $d>2$, $ n > \max \{ 200, (r+2)(d-2)\}$, $r \leq \frac{n}{6}$ and $I = \{1,2,\dots,d-2\}$.  Assume that each column of $\widetilde{\mathbf{\Omega}}_{(I)}$ includes at least $l$ nonzero entries, where 
\begin{eqnarray}\label{minl4j}
l > \max\left\{27 \ \log \left( \frac{2n}{\epsilon} \right) + 9 \ \log \left( \frac{8r(d-2)}{\epsilon} \right) + 18, 6r \right\}. 
\end{eqnarray}
Then, with probability at least $1-\epsilon$, for almost every $\mathcal{U} \in \mathbb{R}^{\overbrace {n \times \dots \times n}^{d}}$, there exist only one completion of the sampled tensor $\mathcal{U}$ with CP rank $r$.
\end{theorem}

\begin{proof}
Similar to the proof of Theorem \ref{mainthmprobfin}, define the $(d-1)$-way tensor $\mathcal{U}^{\prime} \in \mathbb{R}^{\footnotesize {\overbrace {n \times \dots \times n}^{d-2}} \times {n}^2}$ which is obtained through merging the $(d-1)$-th and $d$-th dimensions of the tensor $\mathcal{U}$ and recall that the finiteness of the number of completions of the tensor $\mathcal{U}^{\prime}$ of rank $r$ ensures the finiteness of the number of completions of the tensor $\mathcal{U}$ with rank $r$. Similarly, for simplicity, assume that $\Omega$ and $\breve{\Omega}$ denote the $(d-1)$-way sampling pattern and constraint tensors corresponding to $\mathcal{U}^{\prime}$, respectively. Note that since $n > (r+2)(d-2)$, we conclude $n^2 > r(n-r) + (d-3)r(n-1)+2(d-2)$, and therefore $\widetilde{{\mathbf{\Omega}}}_{{(I)}}$ includes more than  $r(n-r) + (d-3)r(n-1)+2n(d-2)$ columns. According to the proof of Theorem \ref{mainthmprobfin}, considering $r(n-r) + (d-3)r(n-1)$ arbitrary  columns of  $\widetilde{{\mathbf{\Omega}}}_{{(I)}}$ results in existence of ${\breve{\Omega}}^{\prime}$ such that condition (i) holds with probability at least $1 - \frac{\epsilon}{2}$. Also, there exist at least $2n(d-2)$ columns other than these $r(n-r) + (d-3)r(n-1)$ columns.

Consider $n-1$ arbitrary columns of $\widetilde{{\mathbf{\Omega}}}_{{(I)}}$. By setting $r=1$ in the statement of Lemma \ref{lemman3}, these $n-1$ columns result in ${\breve{\Omega}}^{{\prime \prime}^i}$ with $n-1$ columns such that with probability at least $1 - \frac{\epsilon}{4r(d-2)}$, \eqref{ineqpuclm} holds. Similarly, consider $n-r$ arbitrary columns of $\widetilde{{\mathbf{\Omega}}}_{{(I)}}$. Then, there exists ${\breve{\Omega}}^{{\prime \prime}^{i+d-2}}$ such that with probability at least $1 - \frac{\epsilon}{4r(d-2)}$, \eqref{ineqpuclmpr} holds. Hence, condition (ii) holds (for all $i \in \{1,\dots,2d-4\}$) with probability at least $1 - \frac{\epsilon}{2r}$. Therefore, conditions (i) and (ii) hold with probability at least $1- \left( \frac{\epsilon}{2r} + \frac{\epsilon}{2} \right) \geq 1 - \epsilon$.
\end{proof}

\begin{remark}\label{remsamuniq}
A tensor $\mathcal{U}$ that satisfies the properties in the statement of Theorem \ref{mainthmprobuni} requires 
\begin{eqnarray}\label{rremsamuniq}
n^2 \max\left\{27 \ \log \left( \frac{2n}{\epsilon} \right) + 9 \ \log \left( \frac{8r(d-2)}{\epsilon} \right) + 18, 6r \right\}
\end{eqnarray}
samples to be uniquely completable with probability at least $1-\epsilon$, which is orders-of-magnitude lower than the number of samples required by the unfolding approach given in \eqref{unfoldtotsampgen}. Note that the number of samples given in Theorem $3$ of \cite{charact} results in both finite and unique completability, and therefore the number of samples required by the unfolding approach given in Remark \ref{numbofsampmatapp} is for both finite and unique completability.
\end{remark}

\begin{lemma}\label{prsCPuniq}
Assume that $d>2$, $ n > \max \{ 200, (r+2)(d-2) \}$ and $r \leq \frac{n}{6}$. Moreover, assume that the sampling probability satisfies
\begin{eqnarray}\label{prsCPbduniq}
p > \frac{1}{n^{d-2}} \max\left\{27 \ \log \left( \frac{2n}{\epsilon} \right) + 9 \ \log \left( \frac{8r(d-2)}{\epsilon} \right) + 18, 6r \right\} + \frac{1}{\sqrt[4]{n^{d-2}}}
\end{eqnarray}
Then, with probability at least $ (1- \epsilon) \left( 1-\exp(-\frac{\sqrt{n^{d-2}}}{2}) \right)^{n^2}$, $\mathcal{U}$ is finitely completable. 
\end{lemma}

\begin{proof}
The proof is similar to the proof of Lemma \ref{prsCPfin}.
\end{proof}

\section{Numerical Comparisons}\label{simsec}

In order to show the advantage of our proposed CP approach over the unfolding approach, we compare the lower bound on the total number of samples that is required for finite completability using an example. Since the bound on the number of samples for finiteness and uniqueness are the same for the unfolding approach and they are almost the same for the CP approach, we only consider finiteness bounds for this example. In particular, we consider a $7$-way tensor $\mathcal{U} $ ($d=7$) such that each dimension size is $n=10^3$. We also consider the CP rank $r$ which varies from $1$ to $150$. Figure \ref{fig1} plots the bounds given in \eqref{unfoldtotsampgen} (unfolding approach) and in \eqref{rremsam} (CP approach) for the corresponding rank value, where  $\epsilon = 0.001$.  It is seen that the number of samples required by the proposed CP approach is substantially lower than that is required by the unfolding approach.  

\begin{figure}
	\centering
		{\includegraphics[width=11cm]{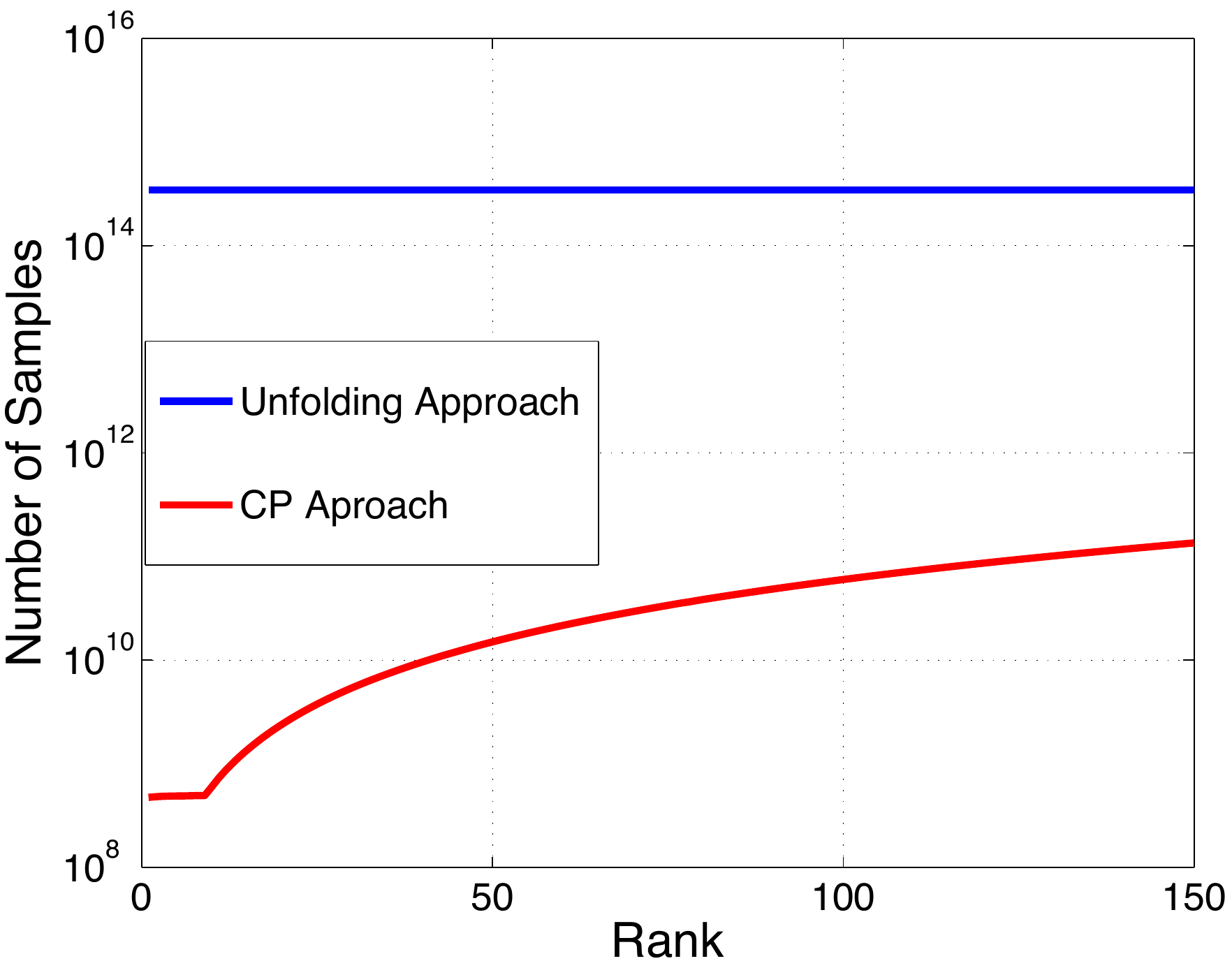}}
	\caption{ Lower bounds on the number of samples for a $7$-way tensor.}
	\label{fig1}\vspace{-4mm}
\end{figure}

\section{Conclusions}\label{concsec}

This paper is concerned with the low CP rank tensor completion problem and aims to derive fundamental conditions on the sampling pattern for finite and unique completability of a sampled tensor given its CP rank. In order to do so, a novel algebraic geometry analysis on the CP manifold is proposed. In particular, each sampled entry can be treated as a polynomials in terms of the entries of the components of the CP decomposition. We have defined a geometric pattern which classifies all CP decompositions such that each class includes only one decomposition of any tensor. We have shown that  finite completability is equivalent to  having a certain number of algebraically independent polynomials among all the defined polynomials based on the sampled entries. Furthermore, using the proposed classification, we can characterize the maximum number of algebraically independent polynomials in terms of a simple function of the sampling pattern. Moreover, we have developed several combinatorial tools that are used to bound the number of samples to ensure finite completability with high probability. Using these developed tools, we have treated three problems in this paper: (i) Characterizing the deterministic necessary and sufficient conditions on the sampling pattern, under which there are only finitely many completions given the CP rank, (ii) Characterizing deterministic sufficient conditions on the sampling pattern, under which there exists exactly one completion given the CP rank, (iii) Deriving lower bounds on the sampling probability or the number of samples such that the obtained deterministic conditions in Problems (i) and (ii) are satisfied with high probability. In addition, it is seen that our proposed CP analysis leads to an orders-of-magnitude  lower number of samples than the unfolding approach that is based on analysis on the Grassmannian manifold.

\bibliographystyle{IEEETran}
\bibliography{bib}

\end{document}